\documentclass{article}

\usepackage{microtype}
\usepackage{graphicx}
\usepackage{subfigure}
\usepackage{booktabs} 

\usepackage{wrapfig}

\usepackage{hyperref}

\usepackage[utf8]{inputenc} 
\usepackage[T1]{fontenc}    
\usepackage{url}            
\usepackage{amsfonts}       
\usepackage{nicefrac}       
\usepackage{subfigure}
\graphicspath{{figures/}}
\usepackage{xcolor}
\usepackage{adjustbox}
\usepackage{bm}
\usepackage{amsmath}
\usepackage{amssymb}

\input{Definitions}

\ifx\proof\undefined
\newenvironment{proof}{\par\noindent{\bf Proof\ }}{\hfill\BlackBox\\[2mm]}
\fi

\ifx\theorem\undefined
\newtheorem{theorem}{Theorem}
\fi
\ifx\example\undefined
\newtheorem{example}{Example}
\fi
\ifx\lemma\undefined
\newtheorem{lemma}[theorem]{Lemma}
\fi

\ifx\corollary\undefined
\newtheorem{corollary}[theorem]{Corollary}
\fi

\ifx\assumption\undefined
\newtheorem{assumption}{Assumption}
\fi

\ifx\definition\undefined
\newtheorem{definition}{Definition}
\fi

\ifx\proposition\undefined
\newtheorem{proposition}[theorem]{Proposition}
\fi

\ifx\remark\undefined
\newtheorem{remark}{Remark}
\fi

\ifx\conjecture\undefined
\newtheorem{conjecture}[theorem]{Conjecture}
\fi

\ifx\factoid\undefined
\newtheorem{factoid}[theorem]{Fact}
\fi

\ifx\axiom\undefined
\newtheorem{axiom}[theorem]{Axiom}
\fi

\newcommand{\RN}[1]{%
	\textup{\lowercase\expandafter{\it \romannumeral#1}}%
}

\usepackage[accepted]{icml2018}

\icmltitlerunning{Continuous-Time Flows}

\begin{document}

\twocolumn[
\icmltitle{Continuous-Time Flows for Efficient Inference and Density Estimation}



\icmlsetsymbol{equal}{*}

\begin{icmlauthorlist}
	\icmlauthor{Changyou Chen}{bu}
	\icmlauthor{Chunyuan Li}{duke}
	\icmlauthor{Liqun Chen}{duke}
	\icmlauthor{Wenlin Wang}{duke}
	\icmlauthor{Yunchen Pu}{duke}
	\icmlauthor{Lawrence Carin}{duke}
\end{icmlauthorlist}

\icmlaffiliation{bu}{SUNY at Buffalo}
\icmlaffiliation{duke}{Duke University}

\icmlcorrespondingauthor{Changyou Chen}{cchangyou@gmail.com}

\icmlkeywords{Machine Learning, ICML}

\vskip 0.3in
]



\printAffiliationsAndNotice{}  

\begin{abstract}
	Two fundamental problems in unsupervised learning are efficient inference for latent-variable models and robust density estimation based on large amounts of unlabeled data. Algorithms for the two tasks, such as normalizing flows and generative adversarial networks (GANs), are often developed independently. In this paper, we propose the concept of {\em continuous-time flows} (CTFs), a family of diffusion-based methods that are able to asymptotically approach a target distribution. Distinct from normalizing flows and GANs, CTFs can be adopted to achieve the above two goals in one framework, with theoretical guarantees. Our framework includes distilling knowledge from a CTF for efficient inference, and learning an explicit  energy-based distribution with CTFs for density estimation. Both tasks rely on a new technique for distribution matching within amortized learning. Experiments on various tasks demonstrate promising performance of the proposed CTF framework, compared to related techniques.
\end{abstract}

\section{Introduction}\label{sec:intro}

Efficient inference and robust density estimation are two important goals in unsupervised learning. In fact, they can be unified from the perspective of learning desired target distributions. In inference problems, one targets to learn a tractable distribution for a {\em latent variable} that is close to a given unnormalized distribution ({\it e.g.}, a posterior distribution in a Bayesian model). In density estimation, one tries to learn an unknown {\em data distribution} only based on the samples from it.  It is also helpful to make a distinction between two types of representations for learning distributions: explicit and implicit methods \citep{MohamedL:arxiv17}. Explicit methods provide a prescribed parametric form for the distribution, while implicit methods learn a stochastic procedure to directly generate samples from the unknown distribution.

Existing deep generative models can easily be identified from this taxonomy. For example, the standard variational autoencoder (VAE) \citep{KingmaW:ICLR14,Rezende:ICML14}  is an important example of an {\it explicit inference} method. Within the inference arm (encoder) of a VAE, recent research has focused on improving the accuracy of the approximation to the posterior distribution on latent variables (codes) using normalizing flow (NF) \citep{RezendeM:ICML15}. NF is particularly interesting due to its ability to approximate the posterior distribution arbitrarily well, while maintaining explicit parametric forms. On the other hand, Stein VAE \citep{FengWL:UAI17,PuGHLHC:NIPS17} is an {\it implicit inference} method, as it only learns to draw samples to approximate posteriors, without assuming an explicit form for the distribution..
For density estimation on observed data, the generative adversarial network (GAN) can be regarded as an {\it implicit density estimation} method \citep{RanganathATB:NIPS16,Huszar:arxiv17,MohamedL:arxiv17}, in the sense that one may sample from the distribution (regarded as a representation of the unknown distribution), but an explicit form for the distribution is not estimated.
GAN has recently been augmented by Flow-GAN~\citep{ermonflow} to incorporate a likelihood term for {\it explicit density estimation}.
Further, the  real-valued non-volume preserving (real NVP) transformations algorithm \citep{dinh2017density} was proposed to perform inference within the implicit density estimation framework.

Some aforementioned methods rely on the concept of {\em flows}.
A flow defines a series of transformations for a random variable (RV), such that the distribution of the RV evolves from a simple distribution to a more complex distribution. When the sequence of transformations are indexed on a discrete-time domain ({\it e.g.}, indexed with integers) with a finite number of transformations, this method is referred to as a normalizing flow \citep{RezendeM:ICML15}. 
Various efficient implementations of NFs have been proposed, such as the planar, radial \citep{RezendeM:ICML15}, Householder \citep{TomczakW:arxiv16}, and inverse autoregressive flows \citep{KingmaSW:NIPS16}. One theoretical limitation of existing normalizing flows is that there is no guarantee on the approximation accuracy due to the finite number of transformations. 

By contrast, little work has explored the applicability of continuous-time flows (CTFs) in deep generative models, where a sequence of transformations are indexed on a continuous-time domain ({\it e.g.}, indexed with real numbers). There are at least two reasons encouraging research in this direction: $\RN{1}$) CTFs are more general than traditional normalizing flows in terms of modeling flexibility, due to the intrinsic infinite number of transformations; $\RN{2}$) CTFs are more theoretically grounded, in the sense that they are guaranteed to approach a target distribution asymptotically (details provided in Section~\ref{sec:VB_CTF}). 

In this paper, we propose efficient ways to apply CTFs for the two motivating tasks. Based on the CTF, our framework learns to drawn samples directly from desired distributions ({\it e.g.}, the unknown posterior and data distributions) for both inference and density estimation tasks via amortization. In addition, we are able to learn an explicit form of the unknown data distribution for density estimation\footnote{Although the density is represented as an energy-based distribution with an intractable normalizer.}. The core idea of our framework is the amortized learning, where knowledge in a CTF is distilled sequentially into another neural network (called {\em inference network} in the inference task, and {\em generator} in density estimation). The distillation relies on the distribution matching technique proposed recently via adversarial lerning \cite{LiLCPCHC:NIPS17}. We conduct various experiments on both synthetic ad real datasets, demonstrating excellent performance of the proposed framework, relative to representative approaches.


\section{Preliminaries}


\subsection{Efficient inference and density estimation}
\vspace{-0.2cm}
\paragraph{Efficient inference with normalizing flows}
Consider a probabilistic generative model with observation $\xb\in\mathbb{R}^D$ and latent variable $\zb\in\mathbb{R}^L$ such that $\xb|\zb \sim p_{\thetab}(\xb | \zb)$ with $\zb \sim p(\zb)$. For efficient inference of $\zb$, the VAE \citep{KingmaW:ICLR14} introduces the concept of an inference network (recognition model or encoder), $q_{\phib}(\zb | \xb)$, as a variational distribution in the VB framework. An inference network is typically a stochastic (nonlinear) mapping from the input $\xb$ to the latent $\zb$, with associated parameters $\phib$. For example, one of the simplest inference networks is defined as $q_{\phib}(\zb | \xb) = \mathcal{N}(\zb; \mub_{\phib}(\xb), \mbox{diag}(\sigmab_{\phib}^2(\xb)))$, where the mean function $\mub_{\phib}(\xb)$ and the standard-derivation function $\sigmab_{\phib}(\xb)$ are specified via deep neural networks parameterized by $\phib$. Parameters are learned by minimizing the negative evidence lower bound (ELBO), {\it i.e.}, the KL divergence between $p_{\thetab}(\xb, \zb)$ and $q_{\phib}(\zb | \xb)$: {\small$\mbox{KL}\left(q_{\phib}(\zb | \xb) \| p_{\thetab}(\xb, \zb)\right) \triangleq \mathbb{E}_{q_{\phib}(\zb | \xb)} \left[\log q_{\phib}(\zb | \xb) - \log p_{\thetab}(\xb, \zb)\right]$}, via stochastic gradient descent \citep{Bottou:12}.

One limitation of the VAE framework is that $q_{\phib}(\zb | \xb)$ is often restricted to simple distributions for feasibility, {\it e.g.}, the normal distribution discussed above, and thus the gap between $q_{\phib}(\zb | \xb)$ and $p_{\thetab}(\zb | \xb)$ is typically large for complicated posterior distributions. NF is a recently proposed VB-based technique designed to mitigate this problem \citep{RezendeM:ICML15}. The idea is to augment $\zb$ via a sequence of deterministic invertible transformations $\{\mathcal{T}_k: \mathbb{R}^L \rightarrow \mathbb{R}^L\}_{k=1}^K$, such that: $\zb_0 \sim q_{\phib}(\cdot | \xb), \zb_1 = \mathcal{T}_1(\zb_0), \cdots, \zb_K = \mathcal{T}_K(\zb_{K-1})$.

Note the transformations $\{\mathcal{T}_k\}$ are typically endowed with different parameters, and we absorb them into $\phib$. Because the transformations are deterministic, the distribution of $\zb_K$ can be written as $q(\zb_K) = q_{\phib}(\zb_0 | \xb)\prod_{k=1}^K\left|\mbox{det}\frac{\partial \mathcal{T}_k}{\partial\zb_k}\right|^{-1}$ via the change of variable formula. As a result, the negative ELBO for normalizing flows becomes:\vspace{-0.2cm}
{\small\begin{align}\label{eq:nf_elbo}
&\mbox{KL}\left(q_{\phib}(\zb_K | \xb) \| p_{\thetab}(\xb, \zb)\right) = 
\mathbb{E}_{q_{\phib}(\zb_0 | \xb)} \left[\log q_{\phib}(\zb_0 | \xb)\right] \\
&- \mathbb{E}_{q_{\phib}(\zb_0 | \xb)} \left[\log p_{\thetab}(\xb, \zb_K)\right] - \mathbb{E}_{q_{\phib}(\zb_0 | \xb)}[\sum_{k=1}^K\log |\mbox{det}\frac{\partial \mathcal{T}_k}{\partial \zb_k}|].\nonumber
\end{align}}
Typically, transformations $\mathcal{T}_k$ of a simple parametric form are employed to make the computations tractable \citep{RezendeM:ICML15}. Our method generalizes these discrete-time transformations to continuous-time ones, ensuring convergence of the transformations to a target distribution.
\vspace{-0.3cm}
\paragraph{Related density-estimation methods}
There exist implicit and explicit density-estimation methods. Implicit density models such as GAN provide a flexible way to draw samples directly from unknown data distributions (via a deep neural network (DNN) called a generator with stochastic inputs) without explicitly modeling their density forms; whereas explicit models such as the pixel RNN/CNN \citep{vandenoordKK:ICML16} define and learn explicit forms of the unknown data distributions. This gives the advantage that the likelihood for a test data point can be explicitly evaluated. However, the generation of samples is typically time-consuming due to the sequential generation nature.

Similar to \citet{WangL:ICLRW17}, our CTF-based approach in Section~\ref{sec:macgan} provides an alternative way for this problem, by simultaneously learning an explicit energy-based data distribution (estimated density) and a generator whose generated samples match the learned data distribution. This not only gives us the advantage of explicit density modeling but also provides an efficient way to generate samples. Note that our technique differs from that of  \citet{WangL:ICLRW17} in that distribution matching is adopted to learn an accurate generator, which is a key component in our framework. 
\vspace{-0.2cm}
\subsection{Continuous-time flows}\label{sec:VB_CTF}
\vspace{-0.1cm}
We notice two potential limitations with traditional normalizing flows: $\RN{1}$) given specified transformations $\{\mathcal{T}_k\}$, there is no guarantee that the distribution of $\zb_K$ could exactly match $p_{\thetab}(\xb, \zb)$; $\RN{2}$) the randomness is only introduced in $\zb_0$ (from the inference network), limiting the representation power. We specify CTFs where transformations are indexed by real numbers, thus they could be considered as consisting of infinite transformations. Further, we consider stochastic flows where randomness is injected in a continuous-time manner. In fact, the concept of CTFs (such as the Hamiltonian flow) has been introduced by \citet{RezendeM:ICML15}, without further development on efficient inference.

We consider a flow on $\mathbb{R}^L$, defined as the mapping\footnote{We reuse the notation $\mathcal{T}$ as transformations from the discrete case above for simplicity, and use $\Zb$ instead of $\zb$ (reserved for the discrete-time setting) to denote the random variable in the continuous-time setting.} $\mathcal{T}: \mathbb{R}^L \times \mathbb{R} \rightarrow \mathbb{R}^L$ such that\footnote{Note we define continuous-time flows in terms of latent variable $\Zb$ in order to incorporate it into the setting of inference. However, the same description applies when we define the flow in data space,  which is the setting of density estimation in Section~\ref{sec:macgan}.} we have $\mathcal{T}(\Zb, 0) = \zb$ and $\mathcal{T}(\mathcal{T}(\Zb, t), s) = \mathcal{T}(\Zb, s + t)$, for all $\Zb \in \mathbb{R}^L$ and $s, t \in \mathbb{R}$. A specific form consider here is defined as $\mathcal{T}(\Zb, t) = \Zb_t$, where $\Zb_t$ is driven by a diffusion of the form:
\vspace{-0.2cm}
\begin{align}\label{eq:diffusion}
\mathrm{d}\Zb_t = F(\Zb_t) \mathrm{d}t + V(\Zb_t)\mathrm{d}\mathcal{W}~.
\end{align}\par\vspace{-0.4cm}
Here $F: \mathbb{R}^L\rightarrow\mathbb{R}^L$, $V: \mathbb{R}^{L\times L} \rightarrow \mathbb{R}^L$ are called the drift term and diffusion term, respectively; $\mathcal{W}$ is the standard $L$-dimensional Brownian motion. In the context of inference, we seek to make the stationary distribution of $\Zb_t$ approach $p_{\thetab}(\zb|\xb)$. One solution for this is to set $F(\Zb_t) = \frac{1}{2}\nabla_{\zb}\log p_{\thetab}(\xb, \zb = \Zb_t)$ and $V(\Zb_t) = \mathbf{I}_L$ with $\mathbf{I}_L$ the $L\times L$ identity matrix. The resulting diffusion is called Langevin dynamics \cite{WellingT:ICML11}. Denoting the distribution of $\Zb_t$ as $\rho_t$, it is well known \cite{Risken:FPE89} that $\rho_t$ is characterized by the Fokker-Planck (FP) equation:
\vspace{-0.2cm}
{\small\begin{align}\label{eq:FPE}
\partial_t \rho_t = -\nabla_{\zb}\cdot (\rho_tF(\Zb_t) + \nabla_{\zb}\cdot(\rho_tV(\Zb_t)V^{\top}(\Zb_t)))~,
\end{align}}\par\vspace{-0.3cm}
where $\ab\cdot\bb \triangleq \ab^{\top} \bb$ for vectors $\ab$ and $\bb$.

For simplicity, we consider the flow defined by the Langevin dynamics specified above, though our results generalize to other stochastic flows \cite{DorogovtsevN:SF14}. In fact, CTF has been applied for scalable Bayesian sampling \cite{DingFBCSN:NIPS14,LICCC:AAAI16,ChenCGLC:AISTATS16,LiSCC:CVPR16,ZhangCHC:ICML17}. In this paper, we generalize it by specifying an ELBO under a CTF, which can then be readily solved by a discretized numerical scheme, based on the results from \citet{JordanKO:MA98}. An approximation error bound for the scheme is also derived. We defer proofs of our theoretical results to the Supplementary Material (SM) for conciseness.
\vspace{-0.2cm}
\section{Continuous-Time Flows for Inference}\label{sec:CTF_inference}
\vspace{-0.2cm}
For this task, we adopt the VAE/normalizing-flow framework with an encoder-decoder structure. An important difference is that instead of feeding data to an encoder and sampling a latent representation in the output as in VAE, we concatenate the data with independent noise as input and directly generate output samples\footnote{Such structure can represent much more complex distributions than a parametric form, useful for the follow up procedures. In contrast, we will define an explicit energy-based distribution for the density in density-estimation tasks.}, constituting an {\em implicit} model. These outputs are then driven by the CTF to approach the true posterior distribution. In the following, we first show that directly optimizing the ELBO is infeasible. We then propose an amortized-learning process that sequentially distills the implicit transformations from the CTF an inference network by distribution matching in Section~\ref{sec:amortize}.

\subsection{The ELBO and discretized approximation}\label{sec:VLB}
We first incorporate CTF into the NF framework by writing out the corresponding ELBO. Note that there are two steps in the inference process. First, an initial $\zb_0$ is drawn from the inference network $q_{\phib}(\cdot | \xb)$; second, $\zb_0$ is evolved via a diffusion such as \eqref{eq:diffusion} for time $T$ (via the transformation $\Zb_T = \mathcal{T}(\zb_0, T)$). Consequently, the negative ELBO for CTF can be written as
{\small\begin{align}\label{eq:ct_elbo}
\mathcal{F}(\xb) &= \mathbb{E}_{q_{\phib}(\zb_0 | \xb)}\mathbb{E}_{\rho_T}\left[\log\rho_T - \log p_{\thetab}(\xb, \Zb_T) \right.\nonumber\\
&\left.+ \log\left|\mbox{det}\frac{\partial \Zb_T}{\partial \zb_0}\right|\right] \triangleq \mathbb{E}_{q_{\phib}(\zb_0 | \xb)}\left[\mathcal{F}_1(\xb, \zb_0)\right]~.
\end{align}}
Note the term $\mathcal{F}_1(\xb, \zb_0)$ is intractable to calculate, in that $\RN{1})$ $\rho_T$ does not have an explicit form; $\RN{2})$ the Jacobian $\frac{\partial \Zb_T}{\partial \zb_0}$ is generally infeasible. In the following, we propose an approximate solution for problem $\RN{1})$. Learning by avoiding problem $\RN{2})$ is presented in Section~\ref{sec:amortize} via amortization. 

For problem $\RN{1})$, a reformulation of the results from \citet{JordanKO:MA98} leads to a nice way to approximate $\rho_t$ in Lemma~\ref{lem:variational_fp}. Note in practice we adopt an {\em implicit} method which uses samples to approximate the solution in Lemma~\ref{lem:variational_fp} for feasibility, detailed in \eqref{eq:ctf_sim}.
\begin{lemma}\label{lem:variational_fp}
	Assume that $\log p_{\thetab}(\xb, \zb)\leq C_1$ is infinitely differentiable, and $\|\nabla_{\zb}\log p_{\thetab}(\xb, \zb)\| \leq C_2\left(1 + C_1 - \log p_{\thetab}(\xb, \zb)\right) (\forall \xb, \zb)$ for some constants $\{C_1, C_2\}$. Let $T = hK$  ($h$ is the stepsize and $K$ is the number of transformations), $\rho_0 \triangleq q_{\phib}(\zb_0 | \xb)$, and $\{\tilde{\rho}_k\}_{k=1}^K$ be the solution of the functional optimization problem:
	\vspace{-0.3cm}\begin{align}\label{eq:variationalFP}
	\tilde{\rho}_k = \arg\min_{\rho \in \mathcal{K}}\mbox{KL}\left(\rho \| p_{\thetab}(\xb, \zb)\right) + \frac{1}{2h}W^2_2\left(\tilde{\rho}_{k-1}, \rho\right)~,
	\end{align}
	with $W_2^2\left(\mu_1, \mu_2\right) \triangleq \inf_{p\in \mathcal{P}(\mu_1, \mu_2)}\int \left\|\xb - \yb\right\|_2^2p(\mathrm{d}\xb,\mathrm{d}\yb)$ the square of 2nd-order Wasserstein distance, and $\mathcal{P}(\mu_1, \mu_2)$ the set of joint distributions on $\{\mu_1, \mu_2\}$. $\mathcal{K}$ is the space of distributions with finite 2nd-order moment. Then $\tilde{\rho}_K$ converges to $\rho_T$ in the limit of $h\rightarrow 0$, {\it i.e.}, $\lim_{h\rightarrow 0}\tilde{\rho}_K = \rho_T$, where $\rho_T$ is the solution of the FP equation \eqref{eq:FPE} at time $T$.
\end{lemma}
Lemma~\ref{lem:variational_fp} reveals an interesting way to compute $\rho_T$ via a sequence of functional optimization problems. By comparing it with the objective of the traditional NF, which minimizes the KL-divergence between $\rho_K$ and $p_{\thetab}(\xb, \zb)$, at each sub-optimization-problem in Lemma~\ref{lem:variational_fp}, it minimizes the KL-divergence between $\tilde{\rho}_k$ and $p_{\thetab}(\xb, \zb)$, plus a regularization term as the Wasserstein distance between $\tilde{\rho}_{k-1}$ and $\tilde{\rho}_k$. The extra Wasserstein-distance term arises naturally due to the fact that the Langevin diffusion can be explained as a gradient flow whose geometry is equipped with the Wasserstein distance~\citep{Otto:ARMA98}. 

The optimization problem in Lemma~\ref{lem:variational_fp} is, however, difficult to deal with directly. In practice, we instead approximate the discretization in an equivalent way by simulation from the CTF. Starting from $\zb_0$, each $\zb_k$ $(k = 0, \cdots, K-1)$ is fed into a transformation $\mathcal{T}_k$ (specified below), resulting in $\zb_{k+1}$ whose distribution coincides with $\tilde{\rho}_{k+1}$ in Lemma~\ref{lem:variational_fp}. The discretization procedure is illustrated in Figure~\ref{fig:discretized}. We must specify the transformations $\mathcal{T}_k$. For each $k$, let $t = hk$; we can conclude from Lemma~\ref{lem:variational_fp} that $\lim_{h\rightarrow 0}\tilde{\rho}_k = \rho_{t}$. From FP theory, $\rho_t$ is obtained by solving the diffusion \eqref{eq:diffusion} with initial condition $\Zb_0 = \zb_0$. It is thus reasonable to specify the transformation $\mathcal{T}_k$ as the $k$-th step of a numerical integrator for \eqref{eq:diffusion}. Specifically, we specify $\mathcal{T}_k$ stochastically:
\begin{align}\label{eq:ctf_sim}
\zb_k = \mathcal{T}_k(\zb_{k-1}) \triangleq \zb_{k-1} + F(\zb_{k-1}) h + V(\zb_{k-1}) \zetab_{k}~,
\end{align}
where $\zetab_k \sim \mathcal{N}(\mathbf{0}, h\mathbf{I}_L)$ is drawn from an isotropic normal. Note the transformation defined here is stochastic, thus we only get samples from $\tilde{\rho}_K$ at the end. A natural way to approximate $\tilde{\rho}_K$ is to use the empirical sample distribution, {\it i.e.}, $\tilde{\rho}_K \approx \frac{1}{K}\sum_{k=1}^K\delta_{\zb_k} \triangleq \bar{\rho}_T$ with $\delta_{\zb}$ a point mass at $\zb$. Afterwards, $\tilde{\rho}_K$ (thus $\bar{\rho}_T$) will be used to approximate the true $\rho_T$ from \eqref{eq:FPE}. 

\begin{figure}[!htb]
	\vskip -0.1in
	\centering
	\includegraphics[width=\linewidth]{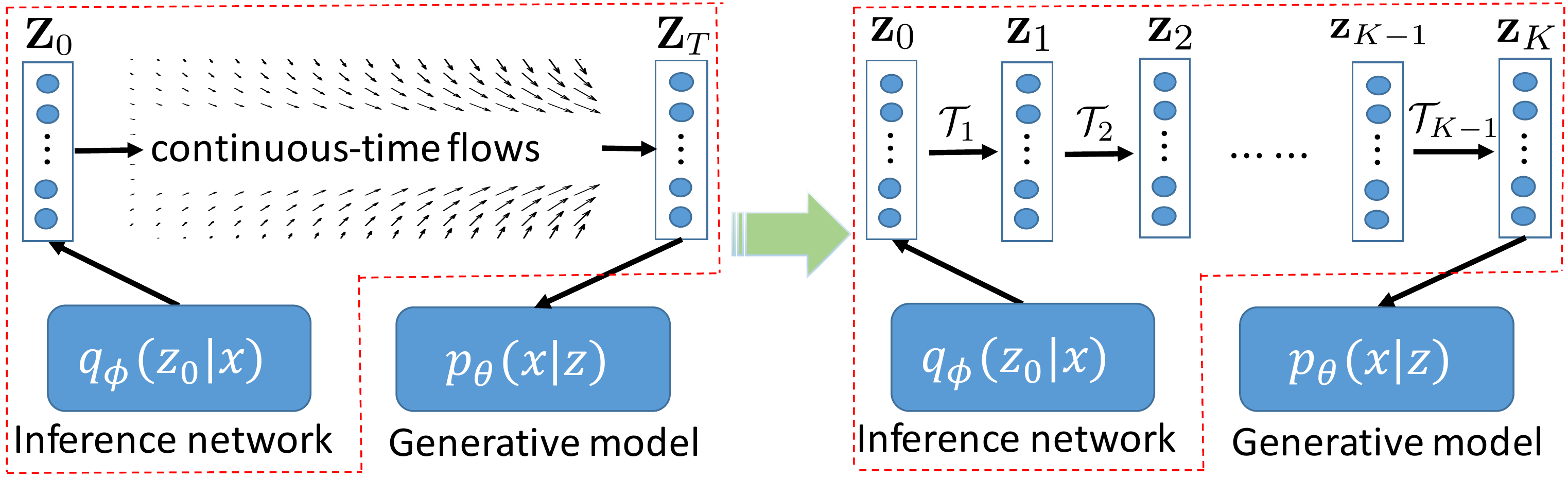}
	\vskip -0.1in
	\caption{Discretized approximation (right) of a continuous-time flow (left). Densities $\{\tilde{\rho}_k\}$ of $\{\zb_k\}$ evolve via transformations $\{\mathcal{T}_k\}$, with $\tilde{\rho}_k \rightarrow \rho_{hk}$ when $h\rightarrow 0$ for each $k$ due to Lemma~\ref{lem:variational_fp}.}
	\label{fig:discretized}
	\vskip -0.1in
\end{figure} 

Better ways to approximate $\rho_T$ might be possible, {\it e.g.}, by assigning more weights to the more recent samples. However, the problem becomes more challenging in theoretical analysis, an interesting point left for future work. In the following, we study how well $\bar{\rho}_T$ approximates $\rho_T$. Following literature on numerical approximation for It\^{o} diffusions \citep{VollmerZT:arxiv15,ChenDC:NIPS15}, we consider a 1-Lipschitz test function $\psi: \mathbb{R}^L \rightarrow \mathbb{R}$, and use the mean square error (MSE) bound to measure the closeness of $\bar{\rho}_T$ and $\rho_T$, defined as: $\mbox{MSE}(\bar{\rho}_T, \rho_T; \psi) \triangleq \mathbb{E}\left(\int \psi(\zb)(\tilde{\rho}_T - \rho_T)(\zb)\mathrm{d}\zb\right)^2$, where the expectation is taken over all the randomness in the construction of $\tilde{\rho}_T$. Note that our goal is related but different from the standard setup as in \cite{VollmerZT:arxiv15,ChenDC:NIPS15}, which studies the closeness of $\bar{\rho}_T$ to $p_{\thetab}(\xb, \zb)$. We need to adopt the assumptions from \citet{VollmerZT:arxiv15,ChenDC:NIPS15}, which are described in the SM. The assumptions are somewhat involved but essentially require coefficients of the diffusion \eqref{eq:diffusion} to be well-behaved. We derive the following bound for the MSE of the sampled approximation, $\bar{\rho}_T$, and the true distribution.

\begin{theorem}\label{theo:mse}
	Under Assumption~\ref{ass:assumption1} in the SM, assume that $\int \rho_T(\zb) p_{\thetab}^{-1}(\xb, \zb)\mathrm{d}\zb < \infty$ and there exists a constant $C$ such that $\frac{\mathrm{d}W_2^2\left(\rho_T, p_{\thetab}(\xb, \zb)\right)}{\mathrm{d} t} \geq C W_2^2\left(\rho_T, p_{\thetab}(\xb, \zb)\right)$, the MSE is bounded as\vspace{-0.2cm}
	{\small\begin{align*}
	\mbox{MSE}(\bar{\rho}_T, \rho_T; \psi) = O\left(\frac{1}{hK} + h^2 + e^{-2ChK}\right)~.
	\end{align*}}
	\vspace{-0.8cm}
\end{theorem}
The last assumption in Theorem~\ref{theo:mse} requires $\rho_T$ to evolve fast through the FP equation, which is a standard assumption used to establish convergence to equilibrium for FP equations \citep{BolleyGG:JFA12}. The MSE bound consists of three terms, the first two terms come from numerical approximation of the continuous-time diffusion, whereas the third term comes from the convergence bound of the FP equation in terms of the Wasserstein distance \citep{BolleyGG:JFA12}. When the time $T = hK$ is large enough, the third term may be ignored due to its exponential-decay rate. Moreover, in the infinite-time limit, the bound endows a bias proportional to $h$; this, however, can be removed by adopting a decreasing-step-size scheme in the numerical method, as in standard stochastic gradient MCMC methods \citep{TehTV:arxiv14,ChenDC:NIPS15}.

\begin{figure*}[!htb]
	\centering
	\includegraphics[width=0.8\linewidth]{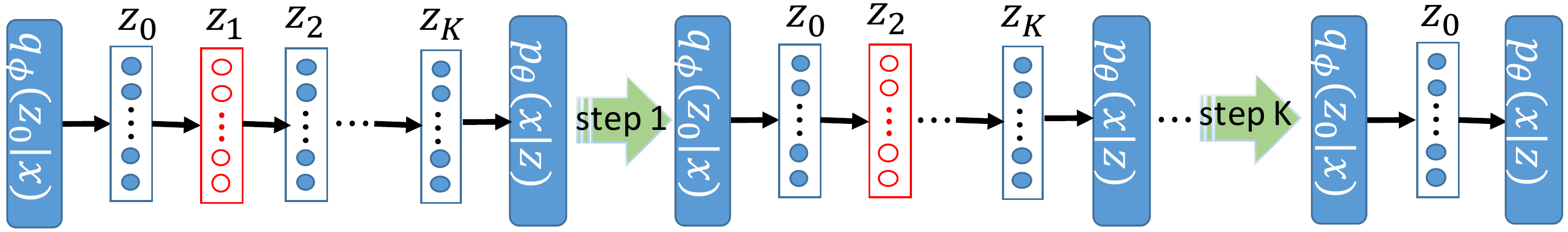}
	\vskip -0.05in
	\caption{Amortized learning of CTFs. From left to right: the initial architecture with $K$-step transformations;  For each step $k$, $q_{\phib}(\cdot)$ is trained to match the distributin of $\zb_k$ in CTFs; In the end, the CTF is distilled into $q_{\phib}(\cdot)$ with distribution matching as in \eqref{eq:update_phi}.}
	\label{fig:amortize}
	\vskip -0.15in
\end{figure*} 
\begin{remark}
	To examine the optimal bound in Theorem~\ref{theo:mse}, we drop out the term $e^{-2ChK}$ in the long-time case (when $hK$ is large enough) for simplicity because it is in a much lower order term than the other terms. The optimal MSE bound (over $h$) decreases at a rate of $O\left(K^{-2/3}\right)$, meaning that $O\left(\epsilon^{-3/2}\right)$ steps of transformations in Figure~\ref{fig:discretized} (right) are needed to reach an $\epsilon$-accurate approximation, {\it i.e.}, $\mbox{MSE} \leq \epsilon$. This is computationally expensive. An efficient way for inference is thus imperative, developed in the next section.
\end{remark}

\subsection{Efficient inference via amortization}\label{sec:amortize}
Even though we approximate $\rho_T$ with $\bar{\rho}_T$, it is still {\em infeasible} to directly apply it to the ELBO in \eqref{eq:ct_elbo} as $\bar{\rho}_T$ is discrete. To deal with this problem, we adopt the idea of ``amortized learning'' \citep{GershmanG:IACCSS14} for inference, by alternatively optimizing the two sets of parameters $\phib$ and $\thetab$.
\paragraph{Updating $\phib$}
To explain the idea, first note that the negative ELBO can be equivalently written as\vspace{-0.2cm}
\begin{align}\label{eq:elbo1}
\mathcal{F}(\xb) = \mathbb{E}_{\rho_0\triangleq q_{\phib}(\zb_0 | \xb)}\mathbb{E}_{\rho_T}\left[\log\rho_0 - \log p_{\thetab}(\xb, \Zb_T)\right]~.
\end{align}
When $\rho_0 = \rho_T$, it is easy to see that: $\mathcal{F}(\xb) = \mathbb{E}_{\rho_0}\left[\log\rho_0 - \log p_{\thetab}(\Zb_T|\xb)\right] + \log p(\xb) = \log p(\xb)$,
which essentially makes the gap between $q_{\phib}(\zb_0 | \xb)$ and $p_{\thetab}(\Zb_T|\xb)$ vanished. As a result, our goal is to learn $\phib$ such that $q_{\phib}(\zb_0 | \xb)$ approaches $p_{\thetab}(\Zb_T|\xb)$. This is a distribution matching problem \cite{LiLCPCHC:NIPS17}. As mentioned previously, we will learn an implicit distribution of $q_{\phib}(\zb_0|\xb)$ ({\it i.e.}, learn how to draw samples from $q_{\phib}(\zb_0|\xb)$ instead of its explicit form), as it allows us to chose a candidate distribution from a much larger distribution space, compared to explicitly defining $q_{\phib}$\footnote{This is distinct from our density-estimation framework described in the next section, where an explicit form is assumed at the beginning for practical needs.}. Consequently, $q_{\phib}(\zb_0|\xb)$ is implemented by a stochastic generator (a DNN parameterized by $\phib$) $Q_{\phib}(\xb, \omega)$ with input as the concatenation of $\xb$ and $\omega$, where $\omega$ is a sample from an isotropic Gaussian distribution $q_0(\omega)$. Our goal is now translated to update the parameter $\phib$ of $Q_{\phib}(\xb, \omega)$ to $\phib^\prime$ such that the distribution of $\{\zb_0^\prime = Q_{\phib^\prime}(\xb, \omega)\}$ with $\omega \sim q_0(\omega)$ matches that of $\zb_1$ in the original generating process with $\phib$ in Figure~\ref{fig:discretized}. In this way, the generating process of $\zb_1$ via $\mathcal{T}_1$ is {\em distilled} into the parameterized generator $Q_{\phib}(\cdot)$, eliminating the need to do a specific transformation via $\mathcal{T}_1$ in testing, and thus is very efficient. Specifically, we update $\phib^\prime$ such that
\vspace{-0.2cm}
\begin{align}\label{eq:update_phi}
\phib^\prime = \arg\min_{\phib}\mathcal{D}(\{\zb_0^{\prime (i)}\}, \{\zb_1^{(i)}\})~,
\end{align}\par \vspace{-0.5cm}
where $\{\zb_0^{\prime (i)}\}_{i=1}^S$ are a set of samples generated from $q_{\phib^\prime}(\zb_0^\prime|\xb)$ via $Q_{\phi}(\cdot)$, and $\{\zb_1^{(i)}\}_{i=1}^S$ are samples drawn by $\omega^i \sim q_0(\omega), \tilde{\zb}_0^i = Q_{\phib}(\xb, \omega^i), \zb_1^{(i)} \sim \mathcal{T}_1(\tilde{\zb}_0^i)$; $\mathcal{D}(\cdot, \cdot)$ is a metric between samples specified below. We call this procedure distilling knowledge from $\mathcal{T}_1$ to $Q_{\phib}(\cdot)$. In practice, one can choose to distill knowledge for several steps ({\it e.g.}, $\mathcal{T}_k$) instead of one step ({\it e.g.}, $\mathcal{T}_1$) to $Q_{\phib}(\cdot)$ each time. Note the distillation idea is related to Bayesian dark knowledge \cite{KorattikaraRMW:NIPS15}, but with different goal and approach.

After distilling knowledge from $\mathcal{T}_1$, we apply the same procedure for other transformations $\mathcal{T}_k  (k > 1)$ sequentially. The final inference network, represented by $q_{\phib}(\cdot | \xb)$, can then well approximate the CTF, {\it e.g.}, the distribution of $\zb_0 \sim q_{\phib}(\cdot | \xb)$ is close to $\rho_T$ from the CTF. This concept is illustrated in Figure~\ref{fig:amortize}. 
We note choosing an appropriate $\mathcal{D}$ in \eqref{eq:update_phi} is important in order to make Theorem~\ref{fig:amortize} applicable. Amortized SVGD \cite{WangL:ICLRW17} defines $\mathcal{D}$ as standard Euclidean distance between samples. We show in Proposition~\ref{prop:D_euc} that this would induce a large error in terms of approximation accuracy.
\begin{proposition}\label{prop:D_euc}
	Fix $\thetab$. If $D$ in \eqref{eq:update_phi} is defined as the summation of pairwise Euclidean distance between samples, then samples generated from $Q_{\phib}$ converge to local modes of $\log p_{\thetab}(\zb|\xb)$.
\end{proposition}
\vspace{-0.2cm}
Consequently, it is crucial to impose more general distance for $\mathcal{D}$. As GAN has been interpreted as distribution matching \cite{LiLCPCHC:NIPS17}, we define $\mathcal{D}$ using the Wasserstein distance, implemented as a discriminator parameterized by a neural network. Specifically, we adopt the ALICE framework \cite{LiLCPCHC:NIPS17}, and use $\{(\xb, \zb_0^{\prime (i)})\}$ as fake data and $\{(\zb_1^{(i)}, \xb^i\sim p_{\thetab}(\cdot|\zb_1^{(i)}))\}$ as real data to train a discriminator. More details are discussed in Section~\ref{supp:discriminator} of the SM.
\paragraph{Updating $\thetab$}
Given $\phib$, $\thetab$ can be updated by simply optimizing the ELBO in \eqref{eq:elbo1}, where $\rho_T$ is approximated by $\bar{\rho}_T$ from the discretized CTF. Specifically, the expectation w.r.t.\! $\rho_T$ in \eqref{eq:elbo1} is approximated by a sample average from:
\vspace{-0.1cm}
{\small\begin{align*}
\zb_0 \sim q_{\phib}(\zb_0 | \xb), \zb_1 \sim \mathcal{T}_1(\zb_0), \zb_2 \sim \mathcal{T}_2(\zb_1), \cdots, \zb_K\sim\mathcal{T}_K(\zb_{K-1})
\end{align*}}
To sum up, there are three steps to learn a CTF-based VAE:
\begin{itemize}\vspace{-0.3cm}
	\item[1.] Generate $(\zb_0, \cdots, \zb_K)$ according to $q_{\phib}(\zb_0 | \xb)$ and the discretized flow with transformations $\{\mathcal{T}_k\}$; \vspace{-0.2cm}
	\item[2.] Update $\phib$ according to \eqref{eq:update_phi};\vspace{-0.2cm}
	\item[3.] Optimize $\thetab$ by minimizing the ELBO \eqref{eq:elbo1} with the generated sample path.\vspace{-0.3cm}
\end{itemize}	
In testing, we use only the finally learned $q_{\phib}(\zb_0 | \xb)$ for inference (into which the CTF has been distilled), and hence testing is like the standard VAE. Since the discretized-CTF model is essentially a Markov chain, we call our model Markov-chain-based VAE (MacVAE).

\section{CTFs for Energy-based Density Estimation}\label{sec:macgan}

We assume that the density of the observation $\xb$ is characterized by a parametric Gibbsian-style probability model $p_{\thetab}(\xb) = \frac{1}{\mathcal{Z}(\thetab)}\tilde{p}_{\thetab}(\xb) \triangleq \frac{1}{\mathcal{Z}(\thetab)}e^{U(\xb; \thetab)}$, where $\tilde{p}_{\thetab}(\xb)$ is an unnormalized version of $p_{\thetab}(\xb)$ with parameter $\thetab$, $U(\xb; \thetab) \triangleq \log \tilde{p}_{\thetab}(\xb)$ is called the energy function \citep{ZhaoML:ICLR17}, and $\mathcal{Z}(\thetab) \triangleq \int \tilde{p}_{\thetab}(\xb) \mathrm{d}\xb$ is the normalizer. Note this form of distributions constitutes a very large class of distributions as long as the capacity of the energy function is large enough. This can be easily achieved by adopting a DNN to implement $U(\xb; \thetab)$, the setting we considered in this paper. Note our model can be placed in between existing implicit and explicit density estimation methods, because we model the data density with an explicit distribution form up to an intractable normalizer. Such distributions have been proved to be useful in real applications, {\it e.g.}, \cite{HaarnojaAL:ICML17} used them to model policies in deep reinforcement learning.

Our goal is to learn $\thetab$ given $\{\xb_i\}_{i=1}^N$, which can be achieved via standard maximum likelihood estimation (MLE): 
$\thetab = \arg\max_{\thetab}\sum_{i=1}^N \log p_{\thetab}(\xb_i) \triangleq \arg\max_{\thetab}\mathcal{M}(\{\xb_i\}; \thetab)$. 

The optimization can be achieved by standard stochastic gradient descent (SGD), with the following gradient formula:
{\small\begin{align}\label{eq:energy_grad}
\frac{\partial \mathcal{M}}{\partial \thetab} = \frac{1}{N}\sum_{i=1}^N\frac{\partial U(\xb_i; \thetab)}{\partial \thetab} - \mathbb{E}_{p_{\thetab}(\xb)}\left[\frac{\partial U(\xb; \thetab)}{\partial \thetab}\right]
\end{align}}
The above formula requires an integration over the model distribution $p_{\thetab}(\xb)$, which can be approximated by Monte Carlo integration with samples. Here we adopt the idea of CTFs and propose to use a DNN guided by a CTF, which we call a {\em generator}, to generate approximate samples from the original model $p_{\thetab}(\xb)$. Specifically, we require that samples from the generator should well approximate the target $p_{\thetab}(\xb)$. This can be done by adopting the CTF idea above, {\it i.e.}, distilling knowledge of a CTF (which approaches $p_{\thetab}(\xb)$) to the generator. In testing, instead of generating samples from $p_{\thetab}(\xb)$ via MCMC (which is complicated and time consuming), we generate samples from the generator directly. Furthermore, when evaluating the likelihood for test data, the constant $\mathcal{Z}(\thetab)$ can also be approximated by Monte Carlo integration with samples drawn from the generator. 

Note the first term on the RHS of \eqref{eq:energy_grad} is a model fit to observed data, and the second term is a model fit to synthetic data drawn from $p_{\thetab(\xb)}$; this is similar to the discriminator in GANs \citep{ArjovskyCB:arxiv17}, but derived directly from the MLE. More connections are discussed below.
\subsection{Learning via Amortization}
\begin{figure}
	\centering
	\includegraphics[width=0.7\linewidth]{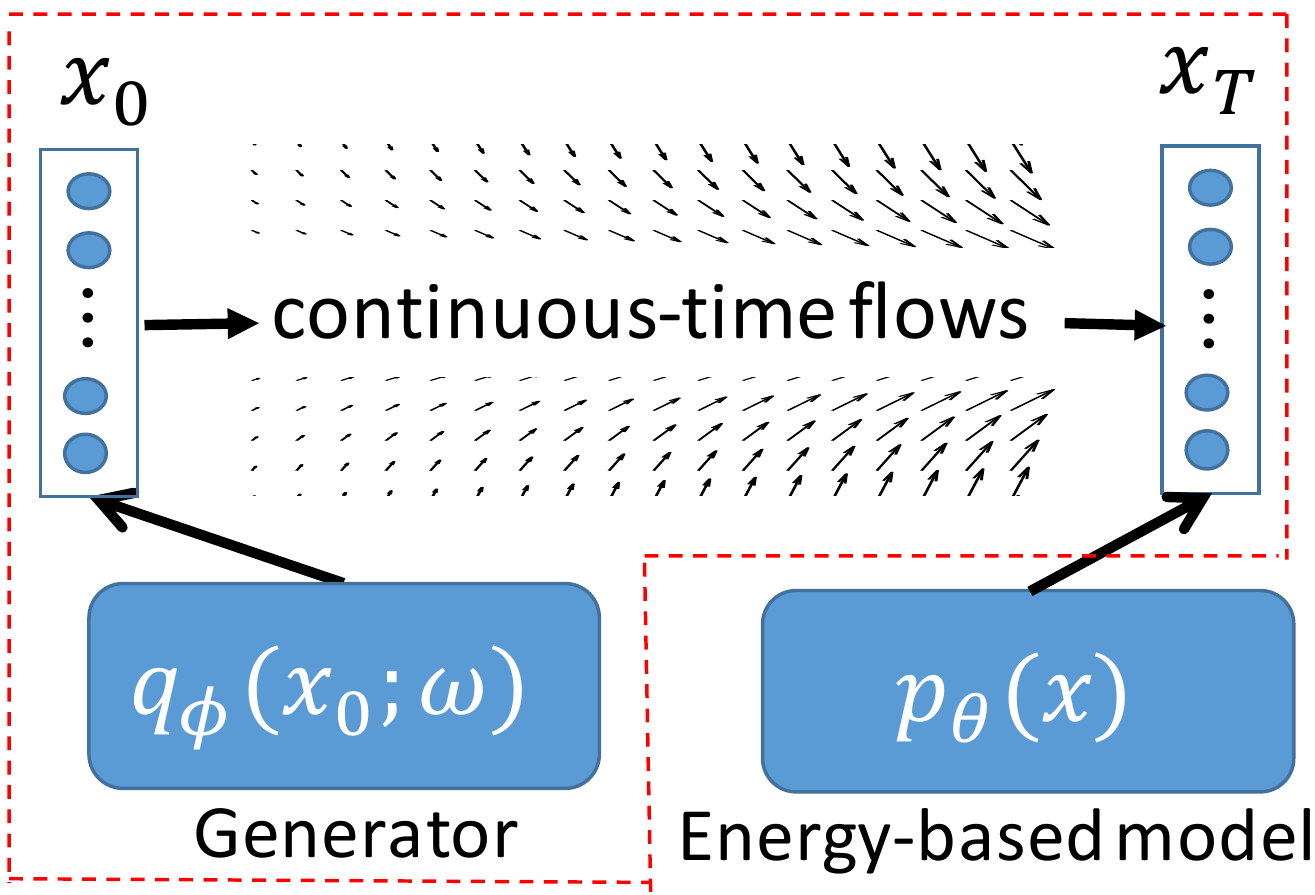}
	\caption{Learning a generator with CTF. The goal is to match the samples $\xb_0$ from $q_{\phib}$ to those after a CTF ($\xb_T$).}
	\label{fig:dgm}
	\vskip -0.25in
\end{figure}
Our goal is to learn a generator whose generated samples match those from the original model $p_{\thetab}(\xb)$.
Similar to inference setting, the generator is learned implicitly. However, we also learn an explicit density model for the data by SGD, with samples from the implicit generator to estimate gradients in \eqref{eq:energy_grad}. Note that in this case, the CTF is performed directly on the data space, instead of on latent-variable space as in previous sections. Specifically, the sampling procedure from the generator plus a CTF are written as: $$\xb_0 \sim q_{\phib}(\xb_0), \xb_T \sim \mathcal{T}(\xb_0, T)~.$$ 
Here $\mathcal{T}(\cdot, \cdot)$ is the continuous-time flow; a sample $\xb_0$ from $q_{\phib}(\cdot)$ is implemented by a deep neural network (generator) $G_{\phib}(\omega)$ with input $\omega \sim q_0(\omega)$,
where $q_0$ is a simple distribution for a noise random variable, {\it e.g.}, the isotropic normal distribution. The procedure is illustrated in Figure~\ref{fig:dgm}. 

Specifically, denote the parameters in the $k$-th step with subscript ``$(k)$''. For efficient sample generation, in the $k$-th step, we again adopt the amortization idea from Section~\ref{sec:amortize} to update $\phib^{(k-1)}$ of the generator network $G_{\phib}(\cdot)$, such that samples from the updated generator match those from the current generator followed by a one-step transformation $\mathcal{T}_1(\cdot)$. After that, $\thetab$ is updated by drawing samples from $q_{\phib}(\cdot)$ to estimate the expectation in \eqref{eq:energy_grad}. The algorithm is presented in Algorithm~\ref{alg:GM} in Section~\ref{supp:DS_CTF} of the SM.

\subsection{Connections to WGAN and MLE}
There is an interesting relation between our model and the WGAN framework \citep{ArjovskyCB:arxiv17}. To see this, let $p_r$ be the data distribution. Substituting $p_{\thetab}(\xb)$ with $q_{\phib}(\xb)$ for the expectation in the gradient formula \eqref{eq:energy_grad} and integrating out $\thetab$, we have that our objective is
{\small\begin{align}\label{eq:macganobj}
	\hspace{-0.2cm}\max \mathbb{E}_{\xb \sim p_r}\left[U(\xb; \thetab)\right] - \mathbb{E}_{\xb \sim q_{\phib}}\left[U(\xb; \thetab)\right]
	\end{align}}
This is an instance of the integral probability metrics \citep{ArjovskyB:ICLR17}. When $U$ is chosen to be 1-Lipschitz functions, it recovers WGAN. This connection motivates us to introduce weight clipping \citep{ArjovskyCB:arxiv17} or alternative regularizers \citep{GulrajaniAADC:arxiv17} when updating $\thetab$ for a better theoretical property. For this reason, we call our model Markov-chain-based GAN (MacGAN). 

Furthermore, it can be shown by Jensen's inequality that the MLE is bounded by (detailed derivations are provided in Section~\ref{app:connection} of the SM)
\vspace{-0.2cm}
{\small\begin{align}\label{eq:lowerbound}
&\max \frac{1}{N}\sum_{i=1}^N \log p_{\thetab}(\xb_i) \\
&\leq \max \mathbb{E}_{\xb \sim p_r}\left[U(\xb; \thetab)\right] - \mathbb{E}_{\xb \sim q_{\phib}} \left[U(\xb; \thetab) \right] - \mathbb{E}_{\xb \sim q_{\phib}} \left[\log q_{\phib}\right]~.\nonumber
\end{align}}
By inspecting \eqref{eq:macganobj} and \eqref{eq:lowerbound}, it is clear that: $\RN{1}$) when learning the energy-based model parameters $\thetab$, the objective can be interpreted as maximizing an upper bound of the MLE shown in \eqref{eq:lowerbound}; $\RN{2}$) when optimizing the parameter $\phib$ of the inference network, we adopt the amortized learning procedure presented in Algorithm~\ref{alg:GM}, whose objective is $\min_{\phib} \mbox{KL}\left(q_{\phib}\|p_{\thetab}\right)$, coinciding with the last two terms in \eqref{eq:lowerbound}. In other words, both $\thetab$ and $\phib$ are optimized by maximizing the {\em same} upper bound of the MLE, guaranteeing convergence of the algorithm, although previous work has pointed out maximizing an upper bound is not a well-posed problem in general \cite{SalakhutdinovH:NC12}. 
\begin{proposition}\label{prop:macgan}
	The optimal solution of MacGAN is the maximum likelihood estimator.
\end{proposition}
Note another difference between MacGAN and standard GAN framework is the way of learning the generator $q_{\phib}$. We adopt the amortization idea, which directly guides $q_{\phib}$ to approach $p_{\thetab}$; whereas in GAN, the generator is optimized via a min-max procedure to make it approach the empirical data distribution $p_r$. By explicitly learning $p_{\thetab}$, MacGAN is able to evaluate likelihood for test data up to a constant.

\section{Related Work}
Our framework extends the idea of normalizing flows \citep{RezendeM:ICML15} and gradient flows \citep{AltieriD:NIPS15} to continuous-time flows, by developing theoretical properties on the convergence behavior. Inference based on CTFs has been studied in \cite{Sohl-DicksteinWMG:ICML15} based on maximum likelihood and \cite{SalimansKW:ICML15} based on the auxiliary-variable technique. However, they directly uses discrete approximations for the flow, and the approximation accuracy is unclear. Moreover, the inference network requires simulating a long Markov chain for the auxiliary model, thus is less efficient than ours. Finally, the inference network is implemented as a parametric distribution ({\it e.g.}, the Gaussian distribution), limiting the representation power, a common setting in existing auxiliary-variable based models \citep{TranRB:ICLR16}.
The idea of amortization \citep{GershmanG:IACCSS14} has recently been explored in various research topics for Bayesian inference such as in variational inference \citep{KingmaW:ICLR14,Rezende:ICML14} and Markov chain Monte Carlo \citep{WangL:ICLRW17,LiTL:ARXIV17,PuGHLHC:NIPS17}. Both \cite{WangL:ICLRW17} and \cite{PuGHLHC:NIPS17} extend the idea of Stein variational gradient descent \citep{LiuW:NIPS16} with amortized inference for a GAN-based and a VAE-based model, respectively, which resemble our proposed MacVAE and MacGAN in concept. \citet{LiTL:ARXIV17} applies amortization to distill knowledge from MCMC to learn a student network. The ideas in \cite{LiTL:ARXIV17} are similar to ours, but the motivation and underlying theory are different from that developed here. The authors proposed several divergence measures for distribution matching including the Jensen-Shannon divergence, similar to our method.

\section{Experiments}
We conduct experiments to test our CTF-based framework for efficient inference and density estimation problems, and compared them with related methods. Some experiments are based on the excellent code for SteinGAN\footnote{\href{https://github.com/DartML/SteinGAN}{https://github.com/DartML/SteinGAN}} \cite{WangL:ICLRW17}, where their default parameter setting are adopted. The discretization stepsize $h$ is robust as long as it is set in a reasonable range, {\it e.g.}, we set it the same as the stepsize in SGD. More experimental results are given in the SM, including a sensitiveness experiment on model parameters in Section~\ref{supp:robust_h}.

\subsection{CTFs for inference}

\vspace{-0.1cm}
\paragraph{Synthetic experiment}
We examine our amortized learning framework with three toy experiments. Particularly, we want to verify the necessity of distribution matching defined in \eqref{eq:update_phi}, {\it i.e.}, we test $\mathcal{D}$ implemented as a discriminator for Wasserstein distance (adversarial-CTF) against that implemented with standard Euclidean distance ($\ell_2$-CTF), which can be considered as an instance of the amortized MCMC \citep{LiTL:ARXIV17} with a Langevin-dynamic transition function and a Euclidean-distance-based divergence measure for samples. Two 2D distributions similar to \cite{RezendeM:ICML15}  are considered, defined in Section~\ref{supp:toy_dist} of the SM.
The inference network $q_{\phib}$ is defined to be a 2-layer MLP with isotropic normal random variables as input. Figure~\ref{fig:demo1} plots the densities estimated with the samples from transformations $\{\mathcal{T}_{K=100}\}$ (before optimizing $\phib$), as well as with samples generated directly from $q_{\phib}$ (after optimizing $\phib$). It is clear that the amortized learning with Wasserstein distance is able to distill knowledge from the CTF to the inference network, while the algorithm fails when Euclidean distance is adopted.

Next, we test MacVAE on a VAE setting on a simple synthetic dataset containing 4 data points, each is a 4D one-hot vector, with the non-zero elements at different positions. The prior of latent code is a 2D standard Normal. Figure~\ref{fig:demo2} plots the distribution of the learned latent code for VAE, adversarial-CTF and $\ell_2$-CTF. Each color means the codes for one particular observation.
It is observed that VAE divides the space into a mixture of 4 Gaussians (consistent with VAE theory), the adversarial-CTF learns complex posteriors, while the $\ell_2$-CTF converges to the mode of each posterior (consistent with Proposition~\ref{prop:D_euc}).
\vspace{-0.2cm}
\paragraph{MacVAE on MNIST}
Following \cite{RezendeM:ICML15,TomczakW:arxiv16}, we define the inference network as a deep neural network with two fully connected layers of size 300 with softplus activation functions. 
We compare MacVAE with the standard VAE and the VAE with normalizing flow, where testing ELBOs are reported (Section~\ref{sec:cal_elbo} of the SM describes how to calculate the ELBO). We do not compare with other state-of-the-art methods such as the inverse autoregressive flow \citep{KingmaSW:NIPS16}, because they typically endowed more complicated inference networks (with more parameters), unfair for comparison. We use the same inference network architecture for all the models. Figure~\ref{fig:inception} (left) plots the testing ELBO versus training epochs. MacVAE outperforms VAE and normalizing flows with a better ELBO (around -85.62).

\begin{figure}[t]
	\centering
	\begin{minipage}{1.0\linewidth}\vspace{-0.0cm}
		\includegraphics[width=0.99\columnwidth]{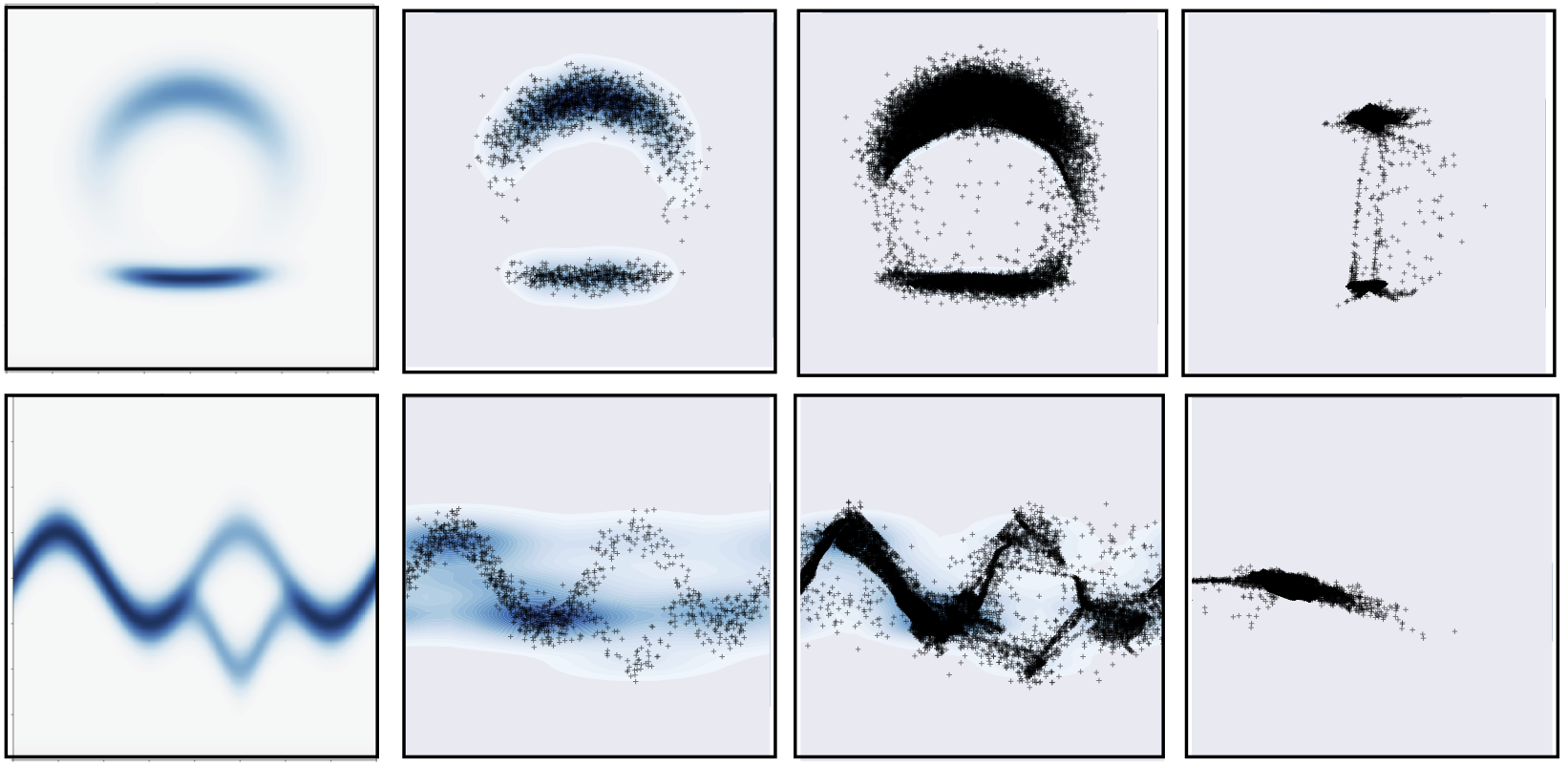}
	\end{minipage}
	\vspace{-3mm}
	\caption{Illustration of CTF on toy distributions. Each row is a distribution case. 1st column: true distributions; 2nd column: MCMC results; 3rd column: approximations via adversarial-CTF; 4th column: approximations via $\ell_2$-CTF.}
	\label{fig:demo1}
\end{figure}

\begin{figure}[ht]
	\centering
	\begin{minipage}{1.0\linewidth}\vspace{-0.2cm}
		\includegraphics[width=0.99\columnwidth]{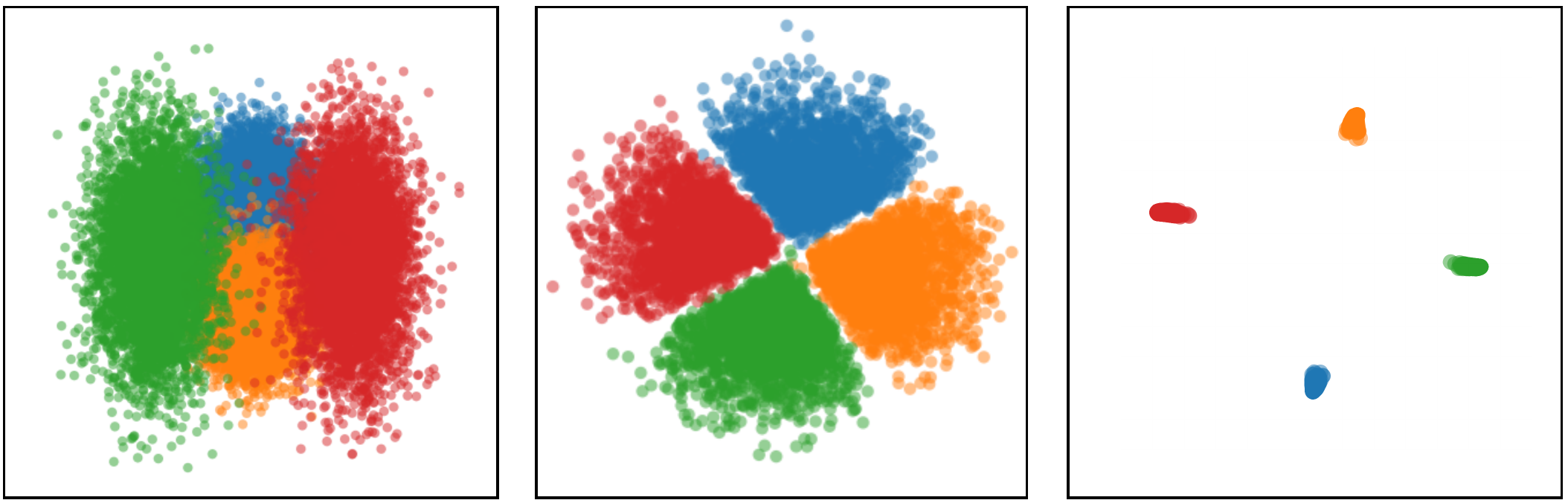}
	\end{minipage}
	\vspace{-3mm}
	\caption{Comparison of the learned latent space with standard VAE (left), adversarial-CTF (middle) and $\ell_2$-CTF (right).}
	\label{fig:demo2}
	\vskip -0.2in
\end{figure}

\begin{figure}[ht]
	\centering
	\includegraphics[width=\linewidth]{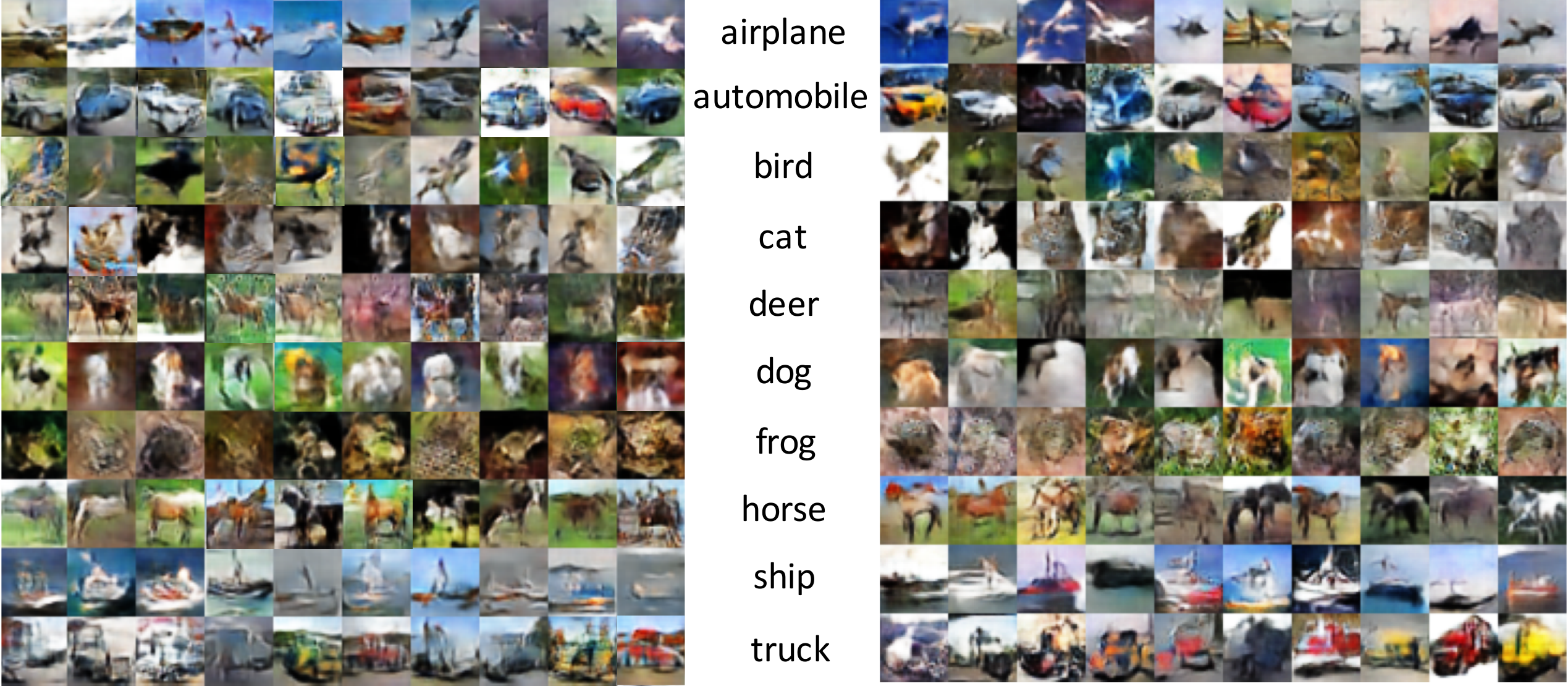}\vspace{0.2cm}
	\includegraphics[width=\linewidth]{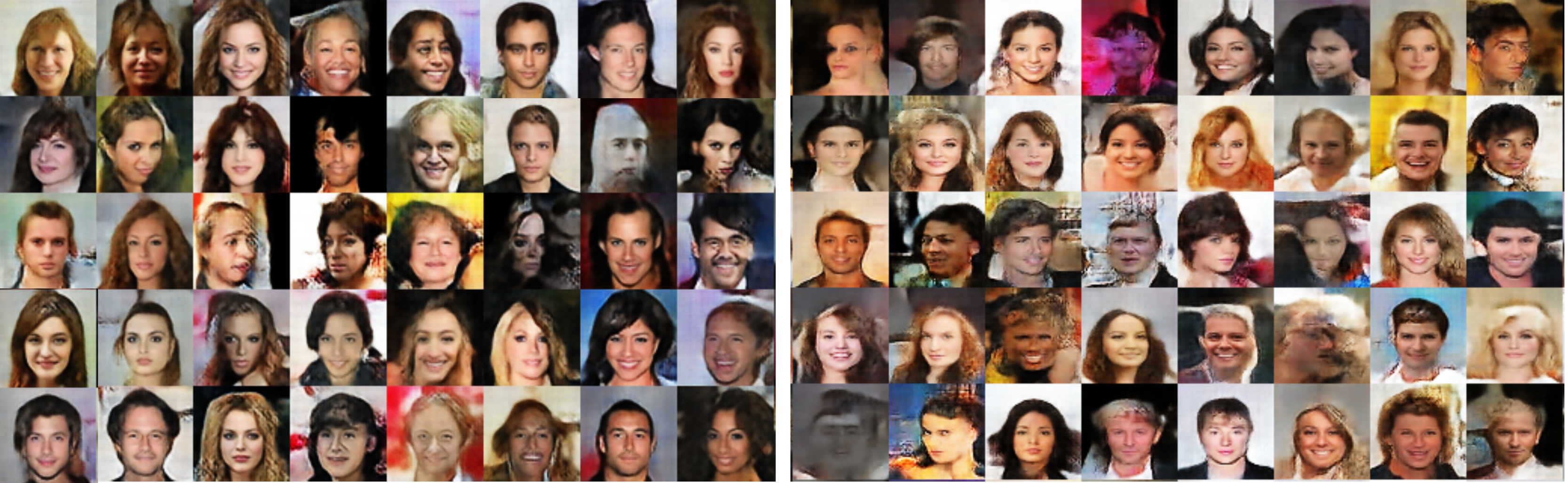}\vspace{0.2cm}
	\includegraphics[width=\linewidth]{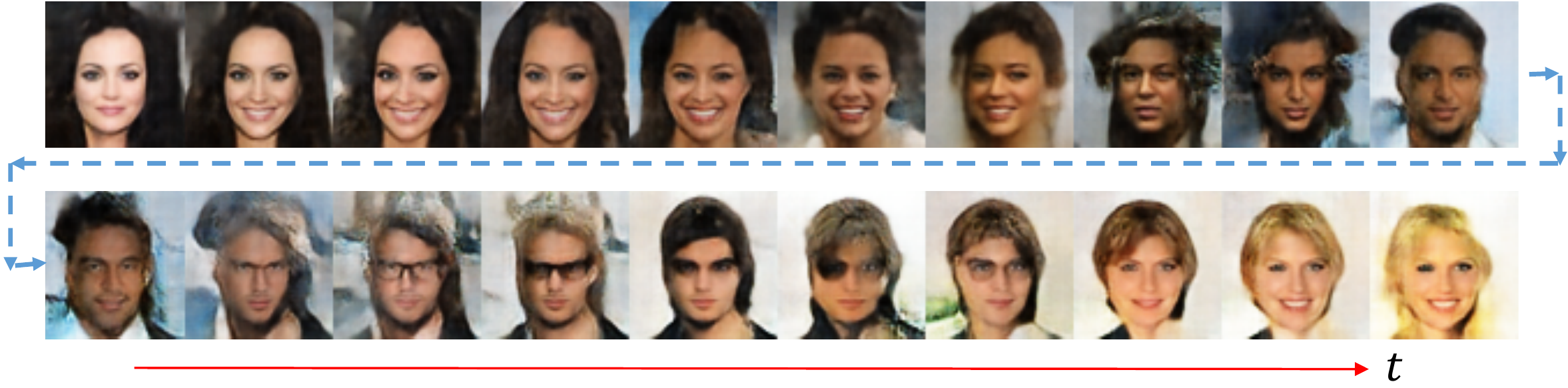}
	\vskip -0.05in
	\caption{Generated images for CIFAR-10 (top) and CelebA (middle) datasets with MacGAN (left) and SteinGAN (right). The bottom are images generated by a random walk on the $\omega$ space for the generator of MacGAN, {\it i.e.}, $\omega_t = \omega_{t-1} + 0.03\times \mbox{rand}([-1, 1])$.}
	\label{fig:macgan_celeb_cifar}
\end{figure} 
\vspace{-0.1cm}
\subsection{CTFs for density estimation}
\vspace{-0.2cm}
We test MacGAN on three datasets: MNIST, CIFAR-10 and CelabA. Following GAN-related methods, the model is evaluated by observing its ability to draw samples from the learned data distribution. Inspiring by \cite{WangL:ICLRW17}, we define a parametric form of the energy-based model as $p_{\thetab}(\xb) \propto \exp\{-\left\|\xb - \mbox{DEC}_{\thetab}\left(\mbox{ENC}_{\thetab}(\xb)\right)\right\|^2\}$, where $\mbox{ENC}_{\thetab}(\cdot)$ and $\mbox{DEC}_{\thetab}(\cdot)$ are encoder and decoder defined by using deep convolutional neural networks and deconvolutional neural networks, respectively, parameterized by $\thetab$. For simplicity, we adopt the popular DCGAN architecture \citep{RadfordMC:arxiv16} for the encoder and decoder. The generator $G_{\phib}$ is defined as a 3-layer CNN with the ReLU activation function (except for the top layer which uses tanh as the activation function, see SM~\ref{sec:additionalexp} for details). Following \cite{WangL:ICLRW17}, the stepsizes are set to $\frac{(m_e - e)\times l_r}{m_e-50}$, where $e$ indexes the epoch, $m_e$ is the total number of epochs, $l_r = \mbox{1e-4}$ when updating $\thetab$, and $l_r = \mbox{1e-3}$ when updating $\phib$. The stepsize in $\mathcal{L}_1$ is set to 1e-3.

\begin{figure}[t]
	\centering
	\begin{minipage}{0.49\linewidth}\vspace{-0.3cm}
		\includegraphics[width=\columnwidth]{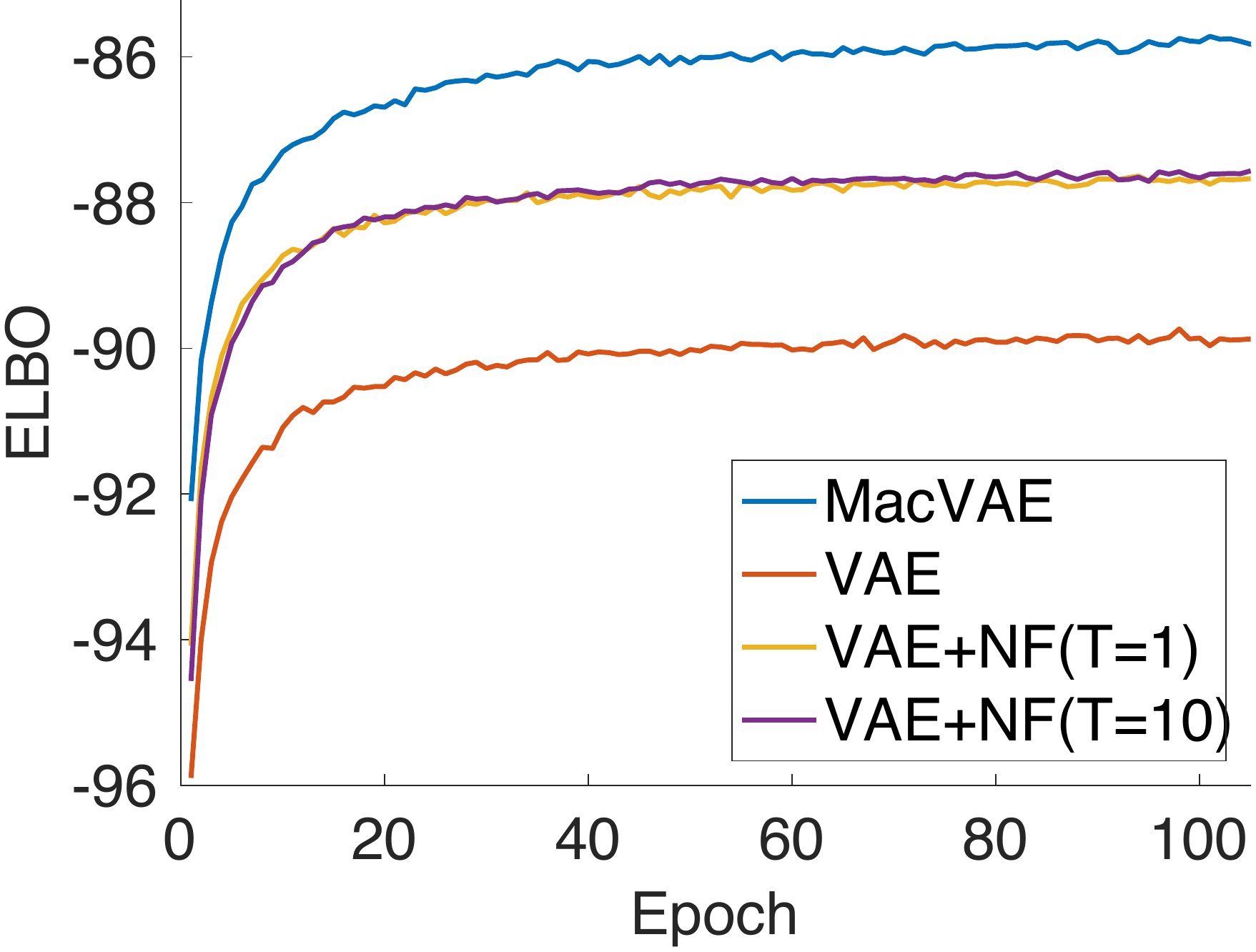}
	\end{minipage}
	\begin{minipage}{0.49\linewidth}
		\includegraphics[width=\columnwidth]{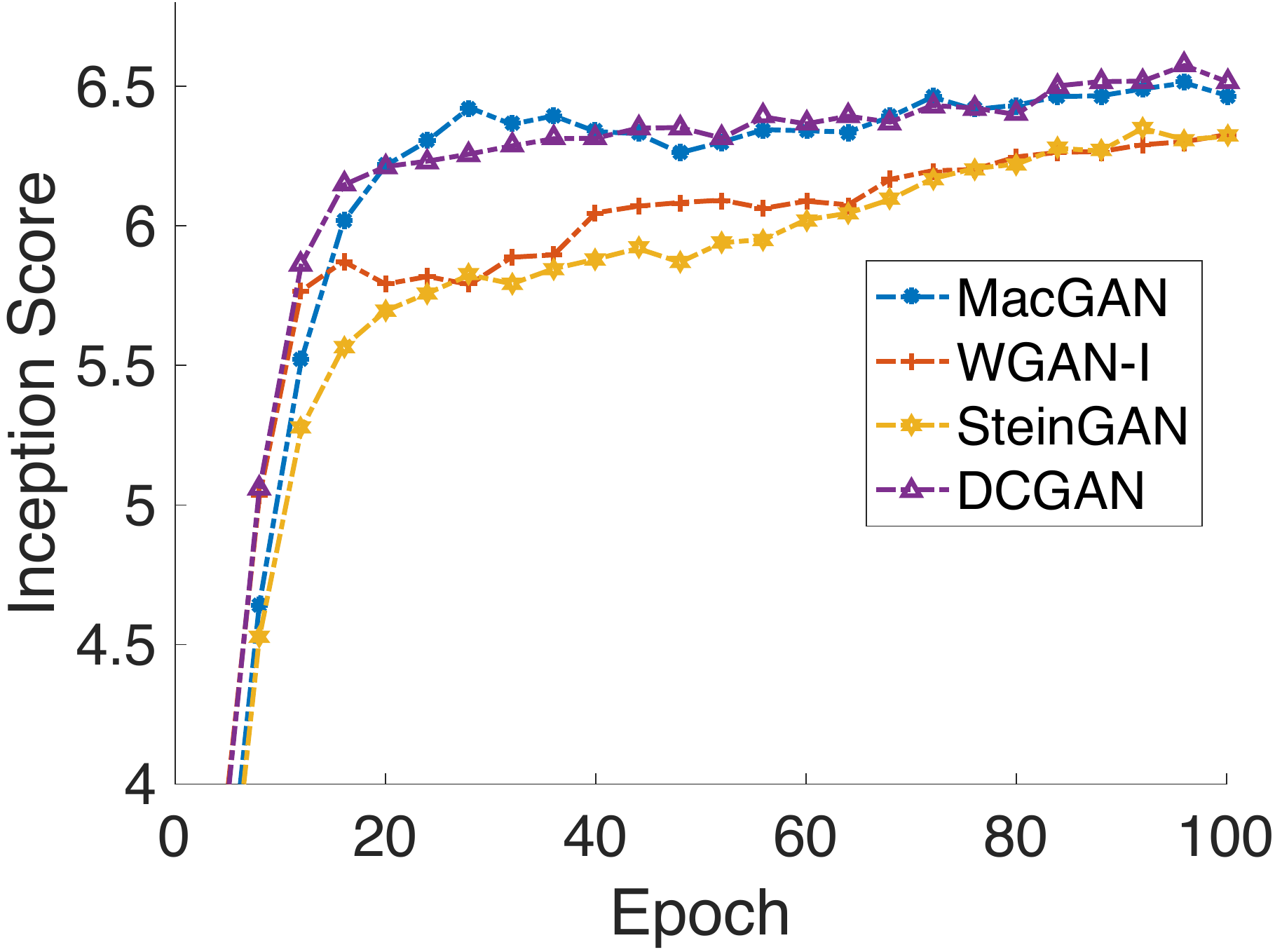}
	\end{minipage}
	\vskip -0.1in
	\caption{ELBO on MNIST vs epochs (left), and Inception score versus epochs (right) for different models. VAE with 80-layer NF is not included because it has much more parameters.}
	\label{fig:inception}
	\vskip -0.2in
\end{figure}
We compare MacGAN with DCGAN \citep{RadfordMC:arxiv16}, the improved WGAN (WGAN-I) \citep{GulrajaniAADC:arxiv17} and SteinGAN \citep{WangL:ICLRW17}. We plot  images generated with MacGAN and its most related method SteinGAN in Figure~\ref{fig:macgan_celeb_cifar} for CelebA and CIFAR-10 datasets. More results are provided in SM Section~\ref{sec:additionalexp}. We observe that visually MacGAN is able to generate clear-looking images. Following \cite{WangL:ICLRW17}, we also plot the images generated by a random walk in the $\omega$ space in Figure~\ref{fig:macgan_celeb_cifar}.

Qualitatively evaluating a GAN-like model is challenging. We follow literature and use the inception score \citep{SalimansGZCRC:arxiv16} to measure the quantity of the generated images. Figure~\ref{fig:inception} (right) plots inception scores vs epochs for different models. MacGAN obtains competitive inception scores with the popular DCGAN model. Quantitatively, we get a final inception score of 6.49 for MacGAN, compared to 6.35 for SteinGAN, 6.25 for WGAN-I and 6.58 for DCGAN.
\vspace{-0.3cm}
\section{Conclusion}
\vspace{-0.2cm}
We study the problem of applying CTFs for efficient inference and explicit density estimation in deep generative models, two important tasks in unsupervised learning. Compared to discrete-time NFs, CTFs are more general and flexible due to the fact that their stationary distributions can be controlled without extra flow parameters. We develop theory on the approximation accuracy when adopting a CTF to approximate a target distribution. We apply CTFs on two classes of deep generative models, a variational autoencoder for efficient inference, and a GAN-like density estimator for explicit density estimation and efficient data generation. Experiments show encouraging results of our framework in both models compared to existing techniques. One interesting direction of future work is to explore more efficient learning algorithms for the proposed CTF-based framework.

\newpage

\section*{Acknowledgements}
We thank the anonymous reviewers for their useful comments. This research was supported in part by DARPA, DOE, NIH, ONR and NSF.
%

\nocite{langley00}

\bibliography{reference}
\bibliographystyle{icml2018}

\newpage\phantom{blabla}
\appendix
%
%
%

\section{Assumptions of Theorem~\ref{theo:mse}}\label{sec:ass}

First, let us define the infinitesimal generator of the diffusion \eqref{eq:diffusion}. Formally, the {\em generator} 
$\mathcal{L}$ of the diffusion \eqref{eq:diffusion} is defined for any compactly supported twice 
differentiable function $f: \RR^L \rightarrow \RR$, such that,
\begin{align*}
&\mathcal{L}f(\Zb_t) =\triangleq \lim_{h \rightarrow 0^{+}} \frac{\mathbb{E}\left[f(\Zb_{t+h})\right] - f(\Zb_t)}{h} \\
&= \rbr{F(\Zb_t) \cdot \nabla + \frac{1}{2}\left(G(\Zb_t) G(\Zb_t)^T\right)\!:\! \nabla \nabla^T} f(\Zb_t)~,
\end{align*}
where $\ab\cdot \bb \triangleq \ab^T\bb$, $\Ab\!:\!\Bb \triangleq \text{tr}(\Ab^T \Bb)$,
$h\rightarrow 0^{+}$ means $h$ approaches zero along the positive real axis.

Given an ergodic diffusion \eqref{eq:diffusion} with an invariant measure 
$\rho(\Zb)$, the posterior average is defined as:
$\bar{\psi} \triangleq \int \psi(\Zb) \rho(\Zb) \mathrm{d}\Zb$ 
for some test function $\psi(\Zb)$ of interest. For a given numerical method with generated samples 
$(\zb_{k})_{k=1}^K$, we use the {\em sample average} $\hat{\psi}$ defined as 
$\hat{\psi}_K = \frac{1}{K} \sum_{k = 1}^K \psi(\zb_{k})$
to approximate $\bar{\psi}$. We define a functional $\tilde{\psi}$ that solves the following \emph{Poisson Equation}:
\begin{align}\label{eq:PoissonEq1}
\mathcal{L} \tilde{\psi}(\zb_{k}) =  \psi(\zb_{k}) - \bar{\psi}
\end{align}

We make the following assumptions on $\tilde{\psi}$.

\begin{assumption}\label{ass:assumption1}
	$\tilde{\psi}$ exists, and its up to 4rd-order derivatives, $\mathcal{D}^k \tilde{\psi}$, are bounded by a
	function $\mathcal{V}$, {\it i.e.}, 
	$\|\mathcal{D}^k \tilde{\psi}\| \leq C_k\mathcal{V}^{p_k}$ for $k=(0, 1, 2, 3, 4)$, $C_k, p_k > 0$. Furthermore, 
	the expectation of $\mathcal{V}$ on $\{\zb_{k}\}$ is bounded: $\sup_l \mathbb{E}\mathcal{V}^p(\zb_{k}) < \infty$, 
	and $\mathcal{V}$ is smooth such that 
	$\sup_{s \in (0, 1)} \mathcal{V}^p\left(s\zb + \left(1-s\right)\yb\right) \leq C\left(\mathcal{V}^p\left(\zb\right) + \mathcal{V}^p\left(\yb\right)\right)$, $\forall \zb, \yb, p \leq \max\{2p_k\}$ for some $C > 0$.
\end{assumption}

\section{Proofs for Section~\ref{sec:CTF_inference}}

\begin{proof}[Sketch Proof of Lemma~\ref{lem:variational_fp}]
	First note that \eqref{eq:variationalFP} in Lemma~\ref{lem:variational_fp} corresponds to eq.13 in \cite{JordanKO:MA98}, where $F(p)$ in \cite{JordanKO:MA98} is in the form of $\mbox{KL}(\rho\|p_{\thetab}(\xb, \zb))$ in our setting. 
	
	Proposition~4.1 in \cite{JordanKO:MA98} then proves that \eqref{eq:variationalFP} has a unique solution. Theorem~5.1 in \cite{JordanKO:MA98} then guarantees that the solution of \eqref{eq:variationalFP} approach the solution of the Fokker-Planck equation in \eqref{eq:FPE}, which is $\rho_T$ in the limit of $h\rightarrow
	0$.
	
	Since this is true for each $k$ (thus each $t$ in $\rho_t$), we conclude that $\tilde{\rho}_k = \rho_{hk}$ in the limit of $h\rightarrow 0$.
\end{proof}

To prove Theorem~\ref{theo:mse}, we first need a convergence result about convergence to equilibrium in Wasserstein distance for Fokker-Planck equations, which is presented in \cite{BolleyGG:JFA12}. Putting in our setting, we can get the following lemma based on Corollary~2.4 in \cite{BolleyGG:JFA12}.

\begin{lemma}[\cite{BolleyGG:JFA12}]\label{lem:convergence_wasserstein}
	Let $\rho_T$ be the solution of the FP equation \eqref{eq:FPE} at time $T$, $p_{\thetab}(\xb, \zb)$ be the joint posterior distribution given $\xb$. Assume that $\int \rho_T(\zb) p_{\thetab}^{-1}(\xb, \zb)\mathrm{d}\zb < \infty$ and there exists a constant $C$ such that $\frac{\mathrm{d}W_2^2\left(\rho_T, p_{\thetab}(\xb, \zb)\right)}{\mathrm{d} t} \geq C W_2^2\left(\rho_T, p_{\thetab}(\xb, \zb)\right)$. Then
	\begin{align}
	W_2\left(\rho_T, p(\xb, \zb)\right) \leq W_2\left(\rho_0, p(\xb, \zb)\right) e^{-CT}~.
	\end{align}
\end{lemma}

We further need to borrow convergence results from \cite{MattinglyST:JNA10,VollmerZT:arxiv15,ChenDC:NIPS15} to characterize error bounds of a numerical integrator for the diffusion \eqref{eq:diffusion}. Specifically, the goal is to evaluate the posterior average of a test function $\psi(\zb)$, defined as $\bar{\psi} \triangleq \int \psi(\zb)p_{\thetab}(\xb, \zb)\mathrm{d}\zb$. When using a numerical integrator to solve \eqref{eq:diffusion} to get samples $\{\zb_k\}_{k=1}^K$, the sample average $\hat{\psi}_K \triangleq \frac{1}{K}\sum_{k=1}^{K}\psi(\zb_k)$ is used to approximate the posterior average. The accuracy is characterized by the mean square error (MSE) defined as: $\mathbb{E}\left(\hat{\psi}_K - \bar{\psi}\right)^2$. Lemma~\ref{lem:discrete} derives the bound for the MSE.

\begin{lemma}[\cite{VollmerZT:arxiv15}]\label{lem:discrete}
	Under Assumption~\ref{ass:assumption1}, and for a 1st-order numerical intergrator, the MSE is bounded, for a constant $C$ independent of $h$ and $K$, by
	\begin{align*}
	\mathbb{E}\left(\hat{\psi}_K - \bar{\psi}\right)^2 \leq C\left(\frac{1}{hK} + h^2\right)~.
	\end{align*}
\end{lemma}

Furthermore, except for the 2nd-order Wasserstein distance defined in Lemma~\ref{lem:variational_fp}, we define the 1st-order Wasserstein distance between two probability measures $\mu_1$ and $\mu_2$ as
\begin{align}\label{eq:W1}
W_1\left(\mu_1, \mu_2\right) \triangleq \inf_{p\in \mathcal{P}(\mu_1, \mu_2)}\int \left\|\xb - \yb\right\|_2p(\mathrm{d}\xb,\mathrm{d}\yb)~.
\end{align}

According to the Kantorovich-Rubinstein duality \cite{ArjovskyCB:arxiv17}, $W_1(\mu_1, \mu_2)$ is equivalently represented as
\begin{align}\label{eq:W1_1}
W_1\left(\mu_1, \mu_2\right) = \sup_{f \in \mathcal{L}_1}\mathbb{E}_{\zb \sim \mu_1}\left[f(\zb)\right] - \mathbb{E}_{\zb \sim \mu_2}\left[f(\zb)\right]~,
\end{align}
where $\mathcal{L}_1$ is the space of 1-Lipschitz functions $f: \mathbb{R}^L \rightarrow \mathbb{R}$.

We have the following relation between $W_1(\mu_1, \mu_2)$ and $W_2(\mu_1, \mu_2)$.

\begin{lemma}[\cite{GivensS:MMJ84}]\label{lem:wasserstein}
	We have for any two distributions $\mu_1$ and $\mu_2$ that $W_1(\mu_1, \mu_2) \leq W_2(\mu_1, \mu_2)$.
\end{lemma}

Now it is ready to prove Theorem~\ref{theo:mse}.

\begin{proof}[Proof of Theorem~\ref{theo:mse}]
	The idea is to simply decompose the MSE into two parts, with one part charactering the MSE of the numerical method, the other part charactering the MSE of $\rho_T$ and $p_{\thetab}(\xb, \zb)$, which consequentially can be bounded using Lemma~\ref{lem:convergence_wasserstein} above.
	
	Specifically, we have
	\begin{align*}
	&\mbox{MSE}(\bar{\rho}_T, \rho_T; \psi) \triangleq \mathbb{E}\left(\int \psi(\zb)(\tilde{\rho}_T - \rho_T)(\zb)\mathrm{d}\zb\right)^2 \\
	=& \mathbb{E}\left(\frac{1}{K}\sum_{k = 1}^K\psi(\zb_k) - \int \psi(\zb)\rho_T(\zb)\mathrm{d}\zb\right)^2 \\
	=& \mathbb{E}\left(\left(\frac{1}{K}\sum_{k = 1}^K\psi(\zb_k) - \int \psi(\zb)p_{\thetab}(\xb, \zb)\mathrm{d}\zb\right) \right. \\
	&\left.~~~~- \left(\int \psi(\zb)\rho_T(\zb)\mathrm{d}\zb - \int \psi(\zb)p_{\thetab}(\xb, \zb)\mathrm{d}\zb\right)\right)^2 \\
	\stackrel{(1)}{=}& \mathbb{E}\left(\frac{1}{K}\sum_{k = 1}^K\psi(\zb_k) - \int \psi(\zb)p_{\thetab}(\xb, \zb)\mathrm{d}\zb\right)^2 \\
	&~~~~+ \left(\int \psi(\zb)\rho_T(\zb)\mathrm{d}\zb - \int \psi(\zb)p_{\thetab}(\xb, \zb)\mathrm{d}\zb\right)^2 \\
	\stackrel{(2)}{\leq}& \mathbb{E}\left(\frac{1}{K}\sum_{k = 1}^K\psi(\zb_k) - \int \psi(\zb)p_{\thetab}(\xb, \zb)\mathrm{d}\zb\right)^2 + W_1^2(\rho_T, p_{\thetab}) \\
	\stackrel{(3)}{\leq}& \mathbb{E}\left(\frac{1}{K}\sum_{k = 1}^K\psi(\zb_k) - \int \psi(\zb)p_{\thetab}(\xb, \zb)\mathrm{d}\zb\right)^2 + W_2^2(\rho_T, p_{\thetab}) \\
	\stackrel{(4)}{\leq}& C_1\left(\frac{1}{hK} + h^2\right) + W_2^2\left(\rho_0, p(\xb, \zb)\right) e^{-2CT} \\
	=& O\left(\frac{1}{hK} + h^2 + e^{-2ChK}\right)~,
	\end{align*}
	where ``(1)'' follows by the fact that $\mathbb{E}\left(\frac{1}{K}\sum_{k = 1}^K\psi(\zb_k) - \int \psi(\zb)p_{\thetab}(\xb, \zb)\mathrm{d}\zb\right) = 0$ \cite{ChenDC:NIPS15}; ``(2)'' follows by the definition of $W_1(\mu_1, \mu_2)$ in \eqref{eq:W1} and the 1-Lipschitz assumption of the test function $\psi$; ``(3)'' follows by Lemma~\ref{lem:wasserstein}; ``(4)'' follows by Lemma~\ref{lem:convergence_wasserstein} and Lemma~\ref{lem:discrete}.
\end{proof}

\section{Sample Distance $\mathcal{D}$ Implemented as a Discriminator in the GAN Framework}\label{supp:discriminator}

We first prove Proposition~\ref{prop:D_euc}, and then describe our implementation for the Wasserstein distance $\mathcal{D}$ in \eqref{eq:update_phi}.

\begin{proof}[Proof of Proposition~\ref{prop:D_euc}]
	By defining $\mathcal{D}$ as standard Euclidean distance, the objective becomes:
	\begin{align*}
	\phib^\prime = \arg\min_{\phib}\frac{1}{S}\sum_{i=1}^S\left\|\zb_0^{\prime (i)} - \zb_1^{(i)}\right\|^2~,
	\end{align*}
	where $\{\zb_0^{\prime (i)}\}_{i=1}^S$ are a set of samples generated from $q_{\phib^\prime}(\zb_0^\prime|\xb)$ via $Q_{\phi}(\cdot)$, {\it i.e.}
	$$\omega^{\prime i} \sim q_0(\omega), ~~\tilde{\zb}_0^{\prime i} = Q_{\phib}(\cdot | \xb, \omega^{\prime i})~,$$
	and $\{\zb_1^{(i)}\}_{i=1}^S$ are samples drawn by $$\omega^i \sim q_0(\omega), ~~\tilde{\zb}_0^i = Q_{\phib}(\cdot | \xb, \omega^i), ~~\zb_1^{(i)} \sim \mathcal{T}_1(\tilde{\zb}_0^i)~.$$
	
	For simplicity, we consider $\mathcal{T}_1$ as one discretized step for Langevin dynamics, {\it i.e.},
	\begin{align*}
		\mathcal{T}_1(\tilde{\zb}_0^i) = \tilde{\zb}_0^i + \nabla_{\zb}\log p_{\thetab}(\xb, \tilde{\zb}_0^i ) h + \sqrt{2h}\xi~,
	\end{align*}
	where $\xi \sim \mathcal{N}(0, \Ib)$. Consequently, the objective becomes
	\begin{align}\label{eq:obj_euc}
		\tilde{F} \triangleq \frac{1}{S}\sum_{i=1}^S&\left\|Q_{\phib}(\cdot | \xb, \omega^{\prime i}) - Q_{\phib}(\cdot | \xb, \omega^i) \right.\nonumber\\
		&\left. - \nabla_{\zb}\log p_{\thetab}(\xb, \tilde{\zb}_0^i ) h + \sqrt{2h}\xi\right\|^2~,
	\end{align}
	\eqref{eq:obj_euc} is a stochastic version of the following equivalent objective:
	\begin{align}\label{eq:obj_euc1}
	F \triangleq &\mathbb{E}_{\omega^{\prime}, \omega \sim p_0(\omega), \xi}\left\|Q_{\phib}(\cdot | \xb, \omega^{\prime i}) - Q_{\phib}(\cdot | \xb, \omega^i) \right.\nonumber\\
	&\left. - \nabla_{\zb}\log p_{\thetab}(\xb, \tilde{\zb}_0^i ) h + \sqrt{2h}\xi\right\|^2~.
	\end{align}
	There are two cases related to $\omega$ and $\omega^{\prime}$. $\RN{1})$ If $\omega$ is restricted to be equal to $\omega^\prime$, {\it e.g.}, they share the same random seed, this is the case in amortized SVGD \cite{WangL:ICLRW17} or amortized MCMC \cite{LiTL:ARXIV17}, as well as in Proposition~\ref{prop:D_euc} where Euclidean distance is adopted. $\RN{2})$ If $\Omega$ and $\Omega^\prime$ do not share the same random seed, this is a more general case, which we also want to show that it can not learn a good generator.
	
	For case $\RN{1})$, $F$ is simplified as:
	\begin{align*}
		F = \mathbb{E}_{\omega\sim p_0(\omega)}\left\|\nabla_{\zb}\log p_{\thetab}(\xb, \tilde{\zb}_0^i )\right\|^2 h^2 + \sqrt{2h}\mathbb{E}_{\xi}\|\xi\|^2~.
	\end{align*}
	Thus the minimum value corresponds to $\nabla_{\zb}\log p_{\thetab}(\xb, \tilde{\zb}_0^i ) = 0$, {\it i.e.}, $\phib$ is updated so that $\tilde{\zb}_0^i$ falls in one of the local modes of $p_{\thetab}(\zb|\xb)$. Proposition~\ref{prop:D_euc} is proved.
	
	We also want to consider case $\RN{2})$. In this case, $F$ is bounded by
	\begin{align*}
		F &\leq \mathbb{E}_{\omega^{\prime}, \omega \sim p_0(\omega)}\left\|Q_{\phib}(\cdot | \xb, \omega^{\prime i}) - Q_{\phib}(\cdot | \xb, \omega^i) \right\|^2 \nonumber\\
		&+ \mathbb{E}_{\omega\sim p_0(\omega)}\left\|\nabla_{\zb}\log p_{\thetab}(\xb, \tilde{\zb}_0^i )\right\|^2 h^2 + 2h\mathbb{E}_{\xi}\|\xi\|^2\nonumber
	\end{align*}
	The minimum possible value of the upper bound of $F$ is achieved when $Q_{\phib}$ matches all $\omega \sim p_0(\omega)$ to a fixed point $\tilde{\zb}$, and also $\nabla_{\zb}\log p_{\thetab}(\xb, \tilde{\zb}) = 0$. This is also a special mode of $p_{\thetab}(\zb|\xb)$ if exists.
	
	To sum up, by defining $\mathcal{D}$ to be standard Euclidean distance, $Q_{\phib}$ would generate samples from local modes of $p_{\thetab}(\zb|\xb)$.
\end{proof}

Now we describe how to define $\mathbb{D}$ as Wasserstein within a GAN framework. Following \cite{LiLCPCHC:NIPS17}, we define a discriminator to match the joint distributions $p(\xb, \zb_{\text{ctf}})$ (an implicit distribution) and $q_{\phi}(\xb, \tilde{\zb})$, where 
\begin{align*}
	q_{\phi}(\xb, \tilde{\zb}) &\triangleq q(\xb)q_{\phi}(\tilde{\zb}|\xb) \\
	(\xb, \zb_{\text{ctf}}) &\sim p(\tilde{\xb}, \zb), \text{ with }\zb_{\text{ctf}} = \mathcal{T}_1(\tilde{\zb})~.
\end{align*}

The graphical structure is defined in Figure~\ref{fig:D_GAN}.
\begin{figure}
	\centering
	\includegraphics[width=0.9\linewidth]{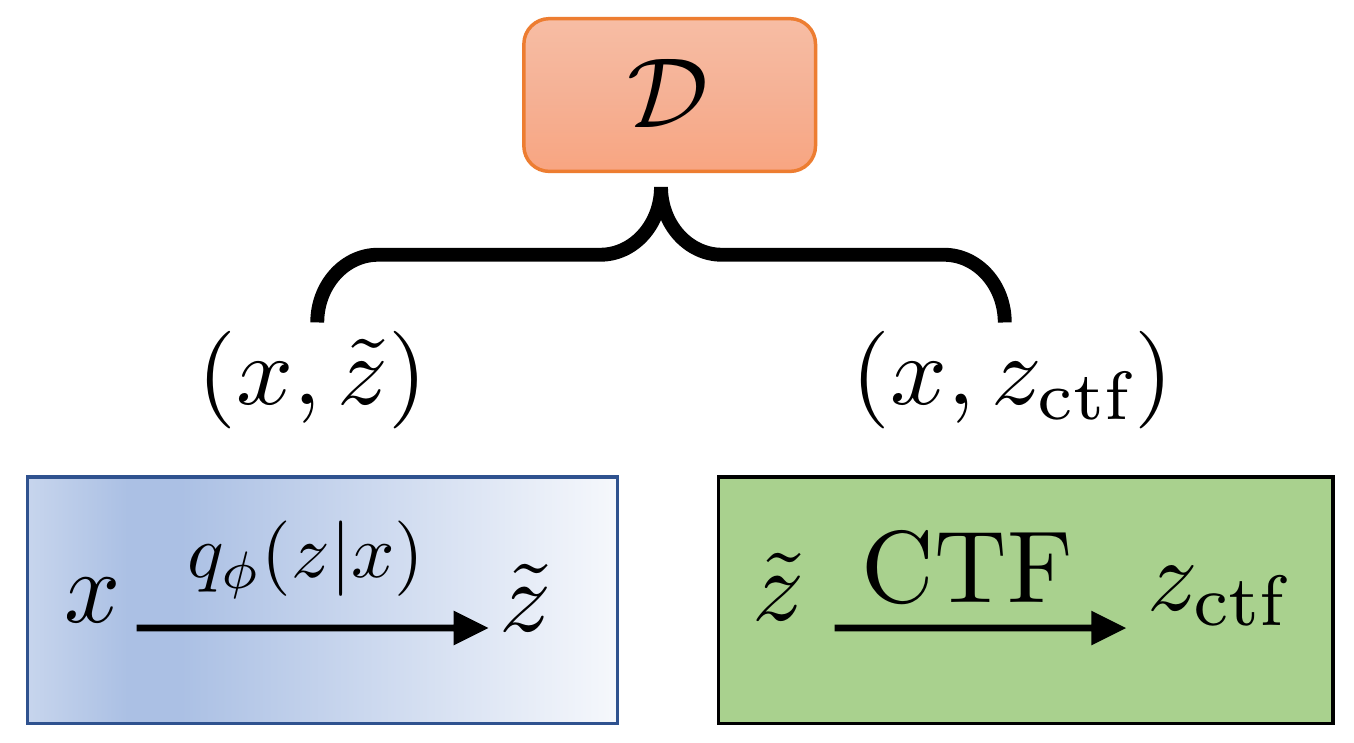}
	\caption{Implementation of $\mathcal{D}$ defined in \eqref{eq:update_phi} for distribution matching with the ALICE framework \cite{LiLCPCHC:NIPS17}.}\label{fig:D_GAN}
\end{figure}

\section{Two 2D Distributions}\label{supp:toy_dist}

$\zb=\{\zb_1, \zb_2\}$: $p(\zb) \propto e^{-U(\zb)}$. 

The first distribution is 
$$U(\zb) \triangleq \frac{1}{2}(\frac{\|\zb\| - 2}{0.4})^2 - \ln (e^{-\frac{1}{2}[\frac{\zb_2 - 4}{2.0}]^2}+e^{-\frac{1}{2}[\frac{\zb_2 + 2}{0.2}]^2})$$ 

The second distribution is 
$$U(\zb) \triangleq - \ln (e^{-\frac{1}{2}[\frac{\zb_2 - w_1(\zb) }{0.35}]^2}+e^{-\frac{1}{2}[\frac{\zb_2 - w_1(\zb) +  w_2(\zb)  }{0.35}]^2})$$ 
where
$$ 
w_2(\zb) = \sin(\frac{2\pi \ab_1}{ 4 }),~~ \text{and}~~w_2(\zb) = 3 \exp( \frac{1}{2}  
\left[ \frac{\zb_1 - 1}{ 0.6 } \right]^2 )
$$

\section{Algorithm for Density Estimation with CTFs}\label{supp:DS_CTF}

Algorithm~\ref{alg:GM} illustrates the details updates for MacGAN.

\begin{algorithm}[t]
	\caption{CTFs for generative models at the $k$-th iteration. $\mathcal{D}(\cdot, \cdot)$ is the same as \eqref{eq:update_phi}.}\label{alg:GM}
	\begin{algorithmic}
		\STATE {\bf Input:} parameters from last step $\thetab^{(k-1)}, \phib^{(k-1)}$
		\STATE {\bf Output:} updated parameters $\thetab^{(k)}, \phib^{(k)}$
		\STATE 1. Generate samples $\{\xb_{1,s}\}_{s=1}^S$ via a discretized CTF: {\small$\xb_{0,s} \sim q_{\phib^{(k-1)}}(\xb_0), \xb_{1,s} \sim \mathcal{T}_1(\xb_{0,s})$};
		\STATE 2. Update the generator by minimizing ($\{\xb_{0,s}^\prime\}_{s=1}^S$ are generated with the updated parameter $\phib^{(k)}$): 
		\vspace{-0.15cm}
		$$\phib^{(k)} = \arg\min_{\phib}\mathcal{D}\left(\{\xb_{1,s}\}, \{\xb_{0,s}^\prime\}\right)~.$$
		\vspace{-0.15cm}
		\STATE 3. Update the energy-based model $\thetab^{k}$ by maximum likelihood, with gradient as \eqref{eq:energy_grad} except replacing $\mathbb{E}_{\xb \sim p_{\thetab}(\xb)}$ with $\mathbb{E}_{\xb \sim q_{\phib}(\xb)}$;
	\end{algorithmic}
\end{algorithm}

\section{Connection to WGAN}\label{app:connection}

We derive the upper bound of the maximum likelihood estimator, which connects MacGAN to WGAN. Let $p_r$ be the data distribution, rewrite our maximum likelihood objective as
\begin{align*}
&\max \frac{1}{N}\sum_{i=1}^N \log p_{\thetab}(\xb_i) \\
&= \max \frac{1}{N}\sum_{i=1}^N \left(U(\xb_i; \thetab) - \log \int e^{U(\xb; \thetab)}\mathrm{d}\xb \right)~.
\end{align*}
The above maximum likelihood estimator can be bounded with Jensen's inequality as: 
\begin{align}\label{eq:mleupper}
&\max \frac{1}{N}\sum_{i=1}^N \log p_{\thetab}(\xb_i) \\
&\leq \max \mathbb{E}_{\xb \sim p_r}\left[U(\xb; \thetab)\right] - \log \int \frac{e^{U(\xb; \thetab)}}{q_{\phib}(\xb; \omega)}q_{\phib}(\xb; \omega)\mathrm{d}\xb \nonumber\\
\leq& \max \mathbb{E}_{\xb \sim p_r}\left[U(\xb; \thetab)\right] - \mathbb{E}_{\xb \sim q_{\phib}(\xb; \omega)} \left[\log \frac{e^{U(\xb; \thetab)}}{q_{\phib}(\xb; \omega)} \right] \nonumber \\
=& \max \mathbb{E}_{\xb \sim p_r}\left[U(\xb; \thetab)\right] - \mathbb{E}_{\xb \sim q_{\phib}(\xb; \omega)} \left[U(\xb; \thetab) \right] \\
&~~~~- \mathbb{E}_{\xb \sim q_{\phib}(\xb; \omega)} \left[\log q_{\phib}(\xb; \omega)\right]~.
\end{align}

This results in the same objective form as WGAN except that our model does not restrict $U(\xb; \thetab)$ to be 1-Lipschitz functions and the objective has an extra constant term $\mathbb{E}_{\xb \sim q_{\phib}(\xb; \omega)} \left[\log q_{\phib}(\xb; \omega)\right]$ w.r.t. $\thetab$. 

Now we prove Proposition~\ref{prop:macgan}.

\begin{proof}[Proof of Proposition~\ref{prop:macgan}]
	First it is clear that the equality in \eqref{eq:mleupper} is achieved if and only if $$q_{\phib}(\xb; \omega) = p_{\thetab}(\xb) \propto e^{U(\xb; \thetab)}~.$$
	
	From the description in Section~\ref{sec:macgan} and \eqref{eq:mleupper}, we know that $\thetab$ and $\phib$ share the same objective function, which is an upper bound of the MLE in \eqref{eq:mleupper}. 
	
	Furthermore, based on the property of continuous-time flows (or formally Theorem~\ref{theo:mse}), we know that $q_{\phib}$ is learned such that $q_{\phib} \rightarrow p_{\thetab}$ in the limit of $h \rightarrow 0$ (or alternatively, we could achieve this by using a decreasing-step-size sequence in a numerical method, as proved in \cite{ChenDC:NIPS15}). When $q_{\phib} = p_{\thetab}$, the equality in \eqref{eq:mleupper} is achieved, leading to the MLE.
\end{proof}

\section{Additional Experiments}\label{sec:additionalexp}

\subsection{Calculating the testing ELBO for MacVAE}\label{sec:cal_elbo}
We follow the method in \cite{PuGHLHC:NIPS17} for calculating the ELBO for a test data $\xb_{*}$. First, after distilling the CTF into the inference network $q_{\phib}$, we have that the ELBO can be represented as
\begin{align*}
\log p(\xb_{*}) \geq \mathbb{E}_{q_{\phib}}\left[\log p_{\thetab}(\xb_{*}, \zb_{*})\right] - \mathbb{E}_{q_{\phib}}\left[\log q_{\phib}\right]~.
\end{align*}
The expectation is approximated with samples $\{\zb_{*j}\}_{j=1}^M$ with $\zb_{*j} = f_{\phib}(\xb_{*}, \zetab_j)$, and $\zetab_j \sim q_0(\zetab)$ the standard isotropic normal. Here $f_{\phib}$ represents the deep neural network in the inference network.
Note $q_{\phib}(\zb_{*})$ is not readily obtained. To evaluate it, we use the density transformation formula: $q_{\phib}(\zb_{*}) = q_0(\zetab)\left|\mbox{det}\frac{\partial f_{\phib}(\xb_{*}, \zetab)}{\partial \zetab}\right|^{-1}$.

\subsection{Network architecture}
The architecture of the generator of MacGAN is given in Table~\ref{tab:macgan_generator}.

\begin{table*}[t]
	\caption{Architecture of generator in MacGAN}
	\label{tab:macgan_generator}
	\centering
	\begin{tabular}{ll}
		\toprule
		Output Size     & Architecture \\
		\midrule
		$100\times 1$ & $100 \times 10$ Linear, BN, ReLU     \\
		$256\times 8 \times 8$    & $512\times 4 \times 4$ deconv, 256 $5\times 5$ kernels, ReLU, strike 2, BN      \\
		$128 \times 16 \times 16$    & $256 \times 8 \times 8$ deconv, 128 $5\times 5$ kernels, ReLU, strike 2, BN      \\
		$3 \times 32\times 32$    & $128 \times 16 \times 16$ deconv, 3 $5\times 5$ kernels, Tanh, strike 2     \\
		\bottomrule
	\end{tabular}
\end{table*}

\subsection{Additional results}
Additional experimental results are given in Figure~\ref{fig:macgan_steingan_mnist} -- \ref{fig:macgan_randomwalk200}.

\begin{figure*}[!htb]
	\centering
	\includegraphics[width=0.9\linewidth]{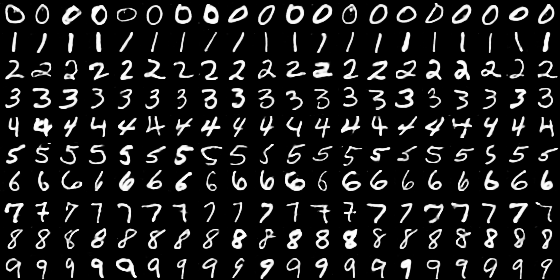}\vspace{0.3cm}
	\includegraphics[width=0.9\linewidth]{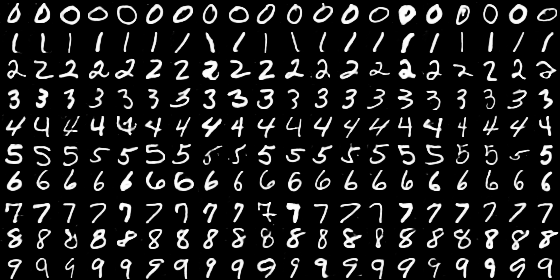}
	\vskip -0.05in
	\caption{Generated images for MNIST datasets with MacGAN (top) and SteinGAN (bottom).}
	\label{fig:macgan_steingan_mnist}
\end{figure*} 

\begin{figure*}[!htb]
	\centering
	\includegraphics[width=0.9\linewidth]{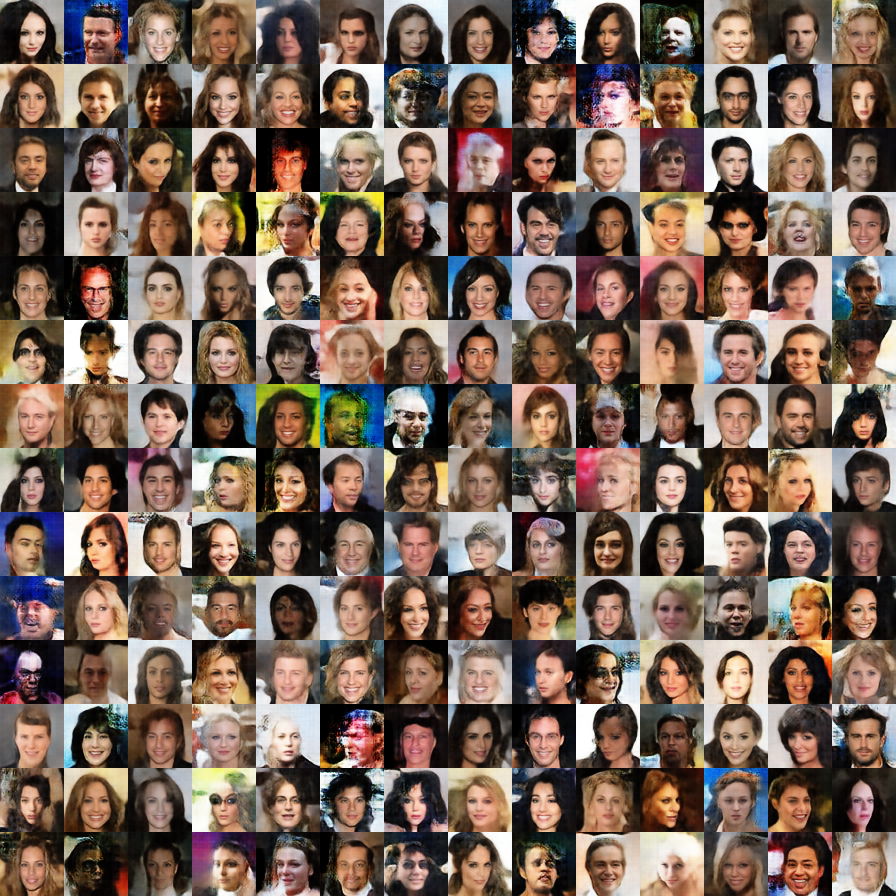}
	\vskip -0.05in
	\caption{Generated images for CelebA datasets with MacGAN.}
	\label{fig:macgan_celeb}
\end{figure*} 

\begin{figure*}[!htb]
	\centering
	\includegraphics[width=0.9\linewidth]{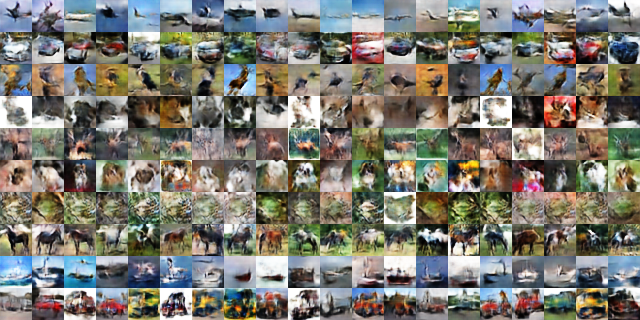}
	\vskip -0.05in
	\caption{Generated images for CIFAR-10 datasets with MacGAN.}
	\label{fig:macgan_cifar}
\end{figure*} 

\begin{figure*}[!htb]
	\centering
	\includegraphics[width=0.9\linewidth]{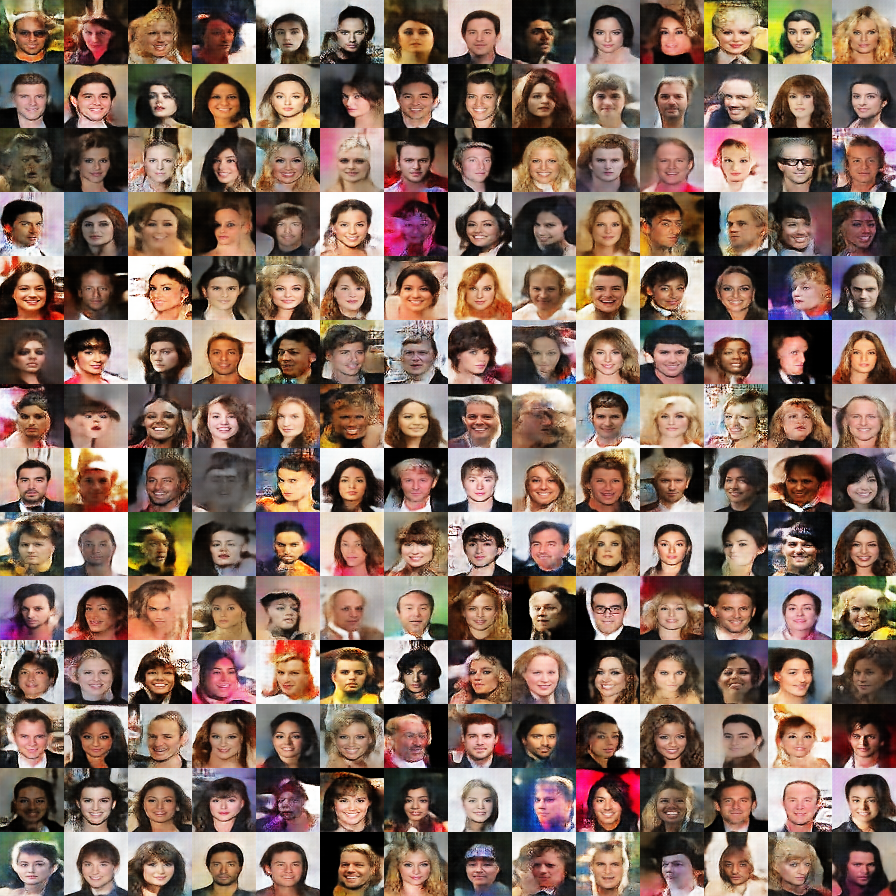}
	\vskip -0.05in
	\caption{Generated images for CelebA datasets with SteinGAN.}
	\label{fig:steingan_celeb}
\end{figure*} 

\begin{figure*}[!htb]
	\centering
	\includegraphics[width=0.9\linewidth]{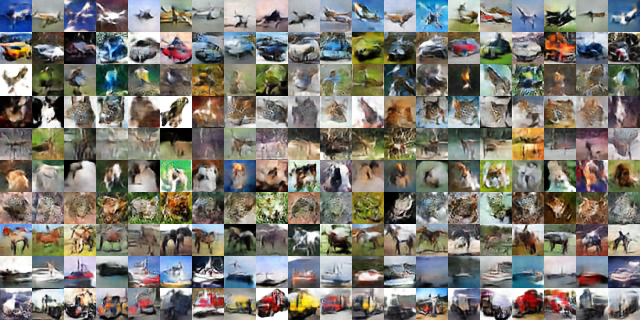}
	\vskip -0.05in
	\caption{Generated images for CIFAR-10 datasets with SteinGAN.}
	\label{fig:steingan_cifar}
\end{figure*} 

\begin{figure*}[!htb]
	\centering
	\includegraphics[width=0.7\linewidth]{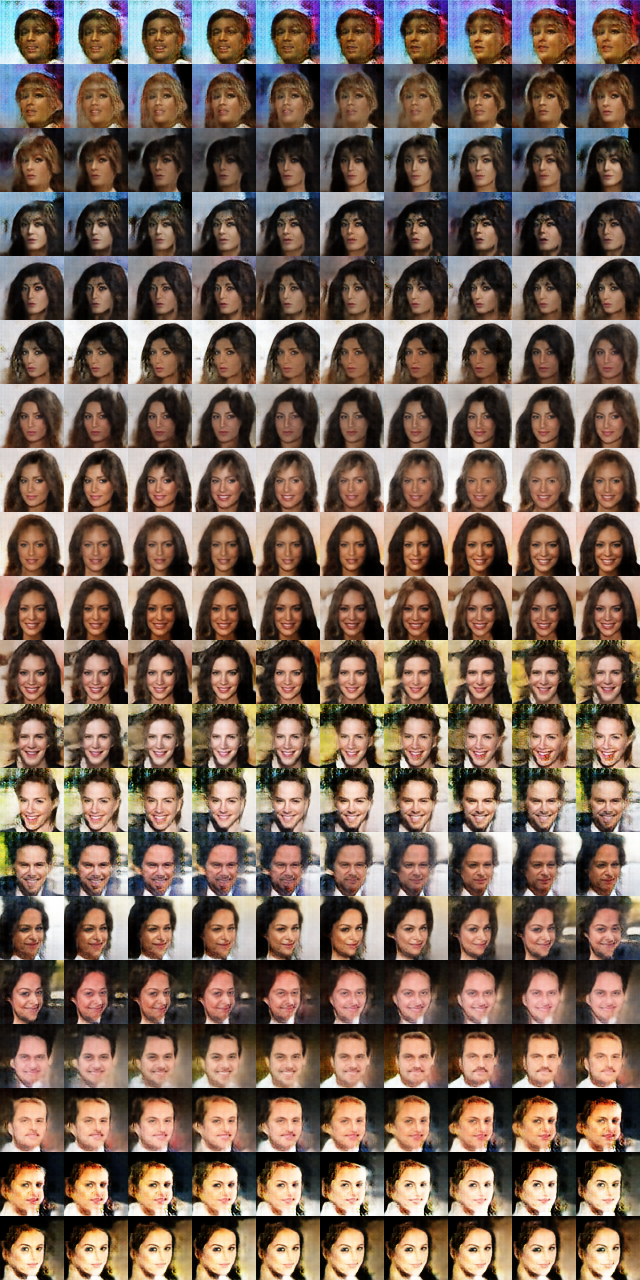}
	\vskip -0.05in
	\caption{Generated images with a random walk on the $\omega$ space for CelebA datasets with MacGAN, $\omega_t = \omega_{t-1} + 0.02 \times \mbox{rand}([-1, 1])$.}
	\label{fig:macgan_randomwalk200}
\end{figure*} 

\subsection{Robustness of the discretization stepsize}\label{supp:robust_h}
To test the impact of the discretization stepsize $h$ in \eqref{eq:ctf_sim}, following SteinGAN \cite{FengWL:UAI17}, we test MacGAN on the MNIST dataset, where ee use a simple Gaussian-Bernoulli Restricted Boltzmann Machines as the energy-based model. We adopt the annealed importance sampling method to evaluate log-likelihoods \cite{FengWL:UAI17}. We vary $h$ in $\{6e-4,2.4e-3,3.6e-3,6e-3,1e-2,1.5e-2\}$. The trend of log-likelihoods is plotted in Figure~\ref{fig:ll_h}. We can see that log-likelihoods do not change a lot within the chosen stepsize interval, demonstrating the robustness of $h$.

\begin{figure}[H]
	\centering
	\includegraphics[width=0.7\linewidth]{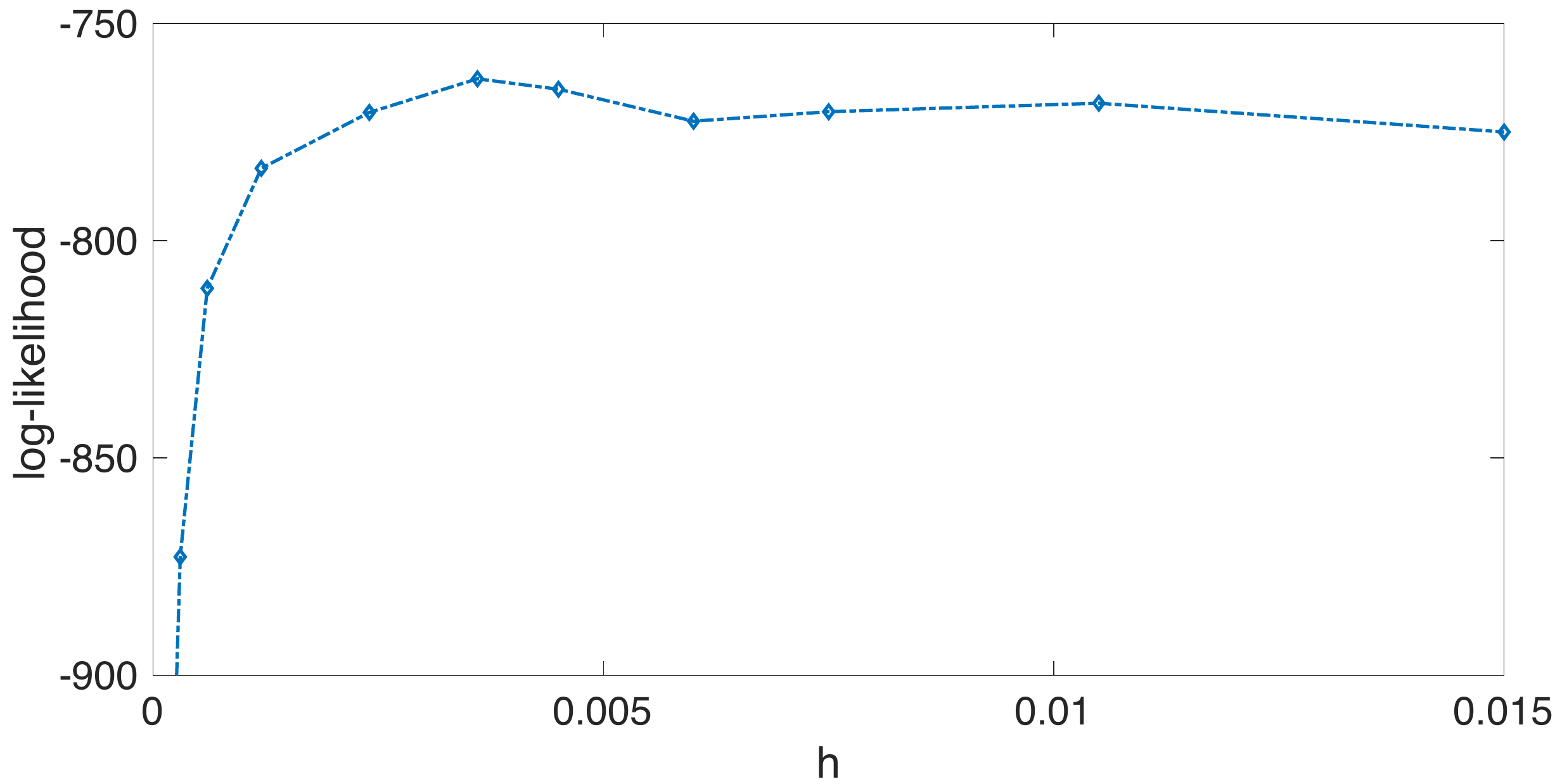}
	\vskip -0.05in
	\caption{Log-likelihoods vs discretization stepsize for MacGAN on MNIST.}
	\label{fig:ll_h}
\end{figure} 


\end{document}